\documentclass[final, 12pt]{alt2021} 

\usepackage{times}

\usepackage[utf8]{inputenc} 
\usepackage[T1]{fontenc}    
\usepackage{hyperref}       
\usepackage{url}            
\usepackage{booktabs}       
\usepackage{nicefrac}       
\usepackage{microtype}      
\usepackage{alias}

\makeatletter
\newcommand\footnoteref[1]{\protected@xdef\@thefnmark{\ref{#1}}\@footnotemark}
\makeatother

\altauthor{%
 \Name{Louis Faury}\footnote[1]{\label{bla}Equal contribution.}
 \Email{l.faury@criteo.com}\\
 \addr  LTCI TélécomParis, Criteo AI Lab
 \AND
 \Name{Yoan Russac}\footnoteref{bla} \Email{yoan.russac@ens.fr}\\
 \addr ENS Paris, Université PSL, CNRS, Inria
 \AND
  \Name{Marc Abeille} \Email{m.abeille@criteo.com}\\
 \addr Criteo AI Lab
 \AND
 \Name{Clément Calauzènes} \Email{c.calauzenes@criteo.com}\\
 \addr Criteo AI Lab
}

\title[Regret Bounds for Generalized Linear Bandits under Parameter Drift]{Regret Bounds for Generalized Linear Bandits under Parameter Drift}

\newif\ifappendix
\appendixfalse

\newif\ifbulletpoint
\bulletpointtrue
\bulletpointfalse

\begin{document}

\maketitle

\begin{abstract}
	Generalized Linear Bandits (GLBs) are powerful extensions to the Linear Bandit (LB) setting, broadening the benefits of reward parametrization beyond linearity. In this paper we study GLBs in non-stationary environments, characterized by a general metric of non-stationarity known as the variation-budget or \emph{parameter-drift}, denoted $B_T$. While previous attempts have been made to extend LB algorithms to this setting, they overlook a salient feature of GLBs which flaws their results. 
	In this work, we introduce a new algorithm that addresses this difficulty. We prove that under a geometric assumption on the action set, our approach enjoys a $\bigo{B_T^{1/3}T^{2/3}}$  regret bound. In the general case, we show that it suffers at most a $\bigo{B_T^{1/5}T^{4/5}}$ regret. At the core of our contribution is a generalization of the projection step introduced in \citet{filippi2010parametric}, adapted to the non-stationary nature of the problem. Our analysis sheds light on central mechanisms inherited from the setting by explicitly splitting the treatment of the learning and tracking aspects of the problem.  
\end{abstract}


\section{Introduction}

\paragraph{Linear Bandits and non-stationarity. }
\ifbulletpoint
\textcolor{red}{The linear bandit is now well understood in non-stationary settings.}
\else
\fi
The Linear Bandit (\lb) framework has proven to be an important paradigm for sequential decision making under uncertainty. It notably extends the Multi-Arm Bandit (\mab) framework to address the exploration-exploitation dilemma when the arm-set is large (potentially infinite) or changing over time. While the \lb{} has now been extensively studied \citep{dani2008stochastic,rusmevichientong2010linearly,abbasi2011improved,abeille2017linear} in its original formulation, a recent strand of research studies its adaptation to non-stationary environments. Notable are the contributions of \citet{cheung2019learning, russac2019weighted, zhao2020simple} which prove that under appropriate algorithmic changes, existing \lb{} concepts can be leveraged to handle a drift of the reward model. Aside their theoretical interests, these results further anchor the spectrum of potential applications of the \lb{} framework to real-world problems, where non-stationarity is commonplace. 

\paragraph{Extensions to Generalized Linear Bandits.} 
\ifbulletpoint
\textcolor{red}{It is important to do the same for GLM. Unfortunately, either false or under weaker measures of non-stationarity.}
\else
\fi
Perhaps the main limitation of \lb{} resides in its inability to model specific (e.g binary, discrete) rewards. One axis of research to operate beyond linearity was initiated with the introduction of Generalized Linear Bandit (\glb{s}) by \cite{filippi2010parametric}. This framework allows to handle rewards which (in expectation) can be expressed as a generalized linear model. Notable members of this family are the logistic and Poisson models. Given the remarkable importance and widespread use of such models in practice, ensuring their resilience to non-stationarity stands as a crucial missing piece. At first glance, as the analysis of \glb{s} mainly relies on tools from the \lb{} literature, one could expect this demonstration to be straight-forward, and almost anecdotal. As a matter of fact, the treatment of GLBs in non-stationary environments was already proposed as a direct extension of non-stationary \lb{} algorithms (\cite[Section 8.3]{cheung2019hedging} and \cite[Section 5.2]{zhao2020simple}). However, as recently pointed out by \citet{russac2020algorithms}, some crucial subtleties of the \glb{s} flaw the analysis and negates the validity of such extensions. An answer to this issue was brought by \citet{russac2020algorithms, russac2020self} who proposed a valid analysis for \glb{s} in non-stationary environments. However, their investigation is restricted to a specific kind of non-stationarity known as \emph{abrupt changes}, leaving the treatment of the superior \emph{parameter-drift} case for future work. To the best of our knowledge, a correct derivation of \glb{s}' behavior under this more general description of non-stationarity is still missing.

\paragraph{Scope and contributions.} 
\ifbulletpoint
\textcolor{red}{we close the gap. provably rate-optimal algorithm. extend the projection step of filippi to non-stationary settings, proves that it is not harder computationally, shed light on some subtleties of the non-stationary glm}
\else
\fi
We focus in this paper on closing this gap. Our main contribution is \textbf{(1)} the design of \ouralgo{} (Algorithm~\ref{alg:algo}), the first \glb{} algorithm resilient to parameter-drift and matching the known
minimax rates - though only for some action sets (Theorem~\ref{thm:regret_bound}). For more general configurations, we still provide a sub-linear regret bound, slightly lagging behind the known rates for non-stationary LBs. Our result relies on \textbf{(2)} a generalization of the projection step of \citet{filippi2010parametric} to non-stationary environments, of similar complexity than its stationary counterpart (Proposition~\ref{prop:program1}). 
Our analysis \textbf{(3)} sheds light on some salient mechanisms of non-stationary bandits.

\section{Preliminaries}
\label{sec:prel}
We consider in this work the stochastic contextual bandit setting under parameter-drift. The environment starts by picking a sequence of parameters $\{\ts{t}\}_{t=1}^\infty$. A repeated game then begins between the environment and an agent. At each round $t$, the environment presents the agent with a set of actions $\mcal{X}_t$ (potentially contextual, large or even infinite). The agent selects an action $x_t\in\mcal{X}_t$ and receives a (stochastic) reward $r_{t+1}$. In this paper we work under the fundamental assumption that there exists a structural relationship between actions and their associated reward in the form of:
\begin{align}
	\mbb{E}\left[r_{t+1}\,\vert \, \mathcal{F}_t , x_t \right]  = \mu\left(\langle x_t,\theta_\star^t\rangle\right).
	\label{eq:reward_def}
\end{align}
The filtration $\mcal{F}_t \defeq \sigma(\{x_s,r_{s+1}\}_{s=1}^{t-1})$ represents the information acquired at round $t$, and $\mu$ is a strictly increasing, continuously differentiable real-valued function most often referred to as the inverse link function. Notable instances of such a problem include the logistic bandit and the Poisson bandit. The goal of the agent is to minimize the cumulative pseudo-regret:
\begin{align*}
	R_T \defeq \sum_{t=1}^{T}\mu(\langle x_\star^t,\ts{t}\rangle)-\mu(\langle x_t,\ts{t}\rangle) \text{ where } x_\star^t = \argmax_{x\in\mcal{X}_t} \mu(\langle x,\ts{t}\rangle)\;.
\end{align*}
We make the following assumption common in the study of parametric bandits:
\begin{ass}[Bounded decision set]
	For all $t\geq 1$, the following holds true: $\ltwo{\ts{t}}\leq S$. Further, the actions have bounded norms: $\ltwo{x}\leq L$ for all $x\in\mathcal{X}_t$.
	\label{ass:bounded_decision_set}
\end{ass}
\begin{ass}[Bounded reward]
	There exists $\sigma>0$ s.t $0\leq r_{t} \leq  2\sigma$ holds almost surely.
	\label{ass:bounded_reward}
\end{ass}
We will denote $\Theta=\{\theta, \ltwo{\theta}\leq S\}$ the set of admissible parameters and $\mcal{X}=\{x, \ltwo{x}\leq L\}$. We assume that the quantities $L$, $S$ and $\sigma$ are known to the agent.
 The true parameters $\{\ts{t}\}_{t=1}^\infty$ are unknown, and their drift is quantified by the variation \emph{variation-budget}, which characterizes the magnitude of the non-stationarity in the environment: 
 \begin{align*}
 B_{T,\star}\defeq \sum_{t=1}^{T-1} \ltwo{\ts{t+1}-\ts{t}} .
\end{align*}
Naturally $B_{T,\star}$ is unknown. For the sake of simplicity and to isolate the main contribution of this paper (\emph{i.e} minimax-optimality in non-stationary GLBs), we will make the following assumption.
\begin{ass}[Variation-budget upper-bound]
    $B_T$ is a \emph{known} quantity such that $B_T\geq B_{T,\star}$.
	\label{ass:B_T_upper_bound}
\end{ass}
This assumption is common in non-stationary bandits
\citep{besbes2014stochastic,cheung2019hedging, zhao2020simple}. We will show in Section~\ref{sec:online_est} how to bypass it with little to no impact on the regret.
For a given inverse link function $\mu$, we will follow the notation from \citet{filippi2010parametric} and denote:
\begin{align*}
	\km=\sup_{x\in\mcal{X}, \theta\in\Theta}\dot\mu\left(\langle x,\theta\rangle\right), \qquad \cm = \inf_{x\in\mcal{X}, \theta\in\Theta} \dot{\mu}\left(\langle x,\theta\rangle\right), \qquad \Rm = \km/\cm \; .
\end{align*}
As in the stationary setting, learning can be canonically performed through the \emph{quasi-maximum likelihood} principle, albeit with adequate modifications. Let $b$ be a primitive of $\mu$. Thanks to the strict increasing nature of the latter, $b$ is a strictly convex function. Let $\lambda>0$ and for $\gamma\in(0,1)$ define\footnote{We follow \citet{russac2019weighted} and use an exponential moving-average strategy.  Our contribution is not specific to this approach and can easily be extended to other alternatives, e.g the sliding window.}:
\begin{align}
	\hat{\theta}_t = \argmin_{\theta\in\mbb{R}^d} \sum_{s=1}^{t-1}\gamma^{t-1-s} \left[b(\langle x_s,\theta\rangle)- r_{s+1}\langle x_s,\theta\rangle \right] + \frac{\lambda\cm}{2}\ltwo{\theta}^2,
	\label{eq:thetahatdef}
\end{align} 
which is well-defined and unique as the minimizer of a strictly convex and coercive function. Further:
\begin{align*}
	g_t(\theta) \defeq \sum_{s=1}^{t-1}\gamma^{t-1-s} \mu(\langle x_s,\theta\rangle)x_s + \lambda\cm \theta.
\end{align*}
Finally, we will use $\mbold{V_t} \defeq \sum_{s=1}^{t-1}\gamma^{t-1-s}x_sx_s^\transp + \lambda \mathbf{I_d} $ and  $\mbold{\widetilde{V}_t} \defeq \sum_{s=1}^{t-1}\gamma^{2(t-1-s)}x_sx_s^\transp + \lambda \mathbf{I_d}$.
Some of our results requires the following assumption on the arm-sets $\mcal{X}_t$. We will discuss the reasons behind this hypothesis, as well as its main implications in the following section. 

\begin{ass}[Orthogonal arm-set]
\label{ass:ortho}
    Let $\{e_i\}_{i=1}^d$ an orthonormal basis of $\mbb{R}^d$. We call a collection of arm-sets $\{\mcal{X}_t\}_t$ \emph{orthogonal} if for all $t\geq 1$ and any $x\in\mcal{X}_t$, there exists $\alpha$ and $i$ such that $x = \alpha e_i$. 
    \label{ass:ortho}
\end{ass}

\section{Related work: limitations and challenges}
\label{sec:related}

\subsection{GLBs and non-stationary LB}

\ifbulletpoint
{\color{red}on GLM, recent interest}
\else
\fi
  \glb{s} were first introduced by \citet{filippi2010parametric} who studied optimistic algorithms which enjoy a $\tilde{\mcal{O}}(\Rm d\sqrt{T})$ regret upper-bound, later refined for $K$-arms problem to $\tilde{\mcal{O}}(\Rm\sqrt{d\log(K)T})$  \citep{li2017provably}. These findings were extended to randomized algorithms, both in the frequentist \citep{abeille2017linear} and Bayesian setting \citep{russo2014learning,dong2018information}. \glb{s} also received an increasing attention targeted at improving their practical implementations \citep{jun2017scalable,dumitrascu2018pg}.
\ifbulletpoint
{\color{red}existing algorithms for non-stationarity in MAB and \lb{}} 
\else
\fi

 Non-stationarity in bandits was first studied in the MAB framework under the specific assumption of abruptly-changing environments (also known as \emph{switching} or \emph{piece-wise stationary} bandits) by \cite{garivier2011upper}. They introduce an algorithm for which they prove $\bigO{\sqrt{\Gamma_TT}}$ regret bounds, where  $\Gamma_T$ is an upper bound on the number of switches.
The effects of the more general parameter-drift were first studied in the MAB setting by \citet{besbes2014stochastic} who for $K$-arm MAB achieved a dynamic regret bound of $\tilde{\mcal{O}}(K^{1/3}B_T^{1/3}T^{2/3})$. Such results were recently extended to the stochastic \lb: \citet{cheung2019learning} developed dynamic policies by resorting to a sliding-window,  \citet{russac2019weighted} introduced a similar approach based on an exponential moving average, and \citet{zhao2020simple} advocated for a simpler restart-based solution. All three aforementioned approaches claim regret bounds of the form $\tilde{\mcal{O}}(d^{2/3}B_T^{1/3}T^{2/3})$, henceforth matching the lower-bound of \citet{cheung2019hedging} up to logarithmic factors. Unfortunately, an error in their analysis was recently pointed out by \cite{touati2020efficient}. It turns out that a correct analysis yields degraded regret bounds, scaling as $\tilde{\mcal{O}}(d^{7/8}B_T^{1/4}T^{3/4})$. This can be improved when the arm sets are orthogonal (Assumption~\ref{ass:ortho}) to retrieve the aforementioned minimax-optimal rates. Note that although this is a rather strong requirement, it does not reduce to MAB as it still allows for infinite and changing arm-sets. 

\subsection{Toward non-stationary GLBs: limitations}
\label{subsec:limitations}
\ifbulletpoint
{\color{red} and extensions under non-stationarity}
\else
\fi

\paragraph{On the limits of piece-wise stationarity.} 
\ifbulletpoint
{\color{red}
\begin{itemize}
	\item Budget variation vs abrupt change
	\begin{itemize}
		\item In Russac, far from switch = stationary
		\item For us, fundamental tension between learning and tracking comes into play
	\end{itemize}
\end{itemize}
}
\else
\fi
To the best of our knowledge, the first valid analysis of non-stationary GLBs was conducted by \citet{russac2020algorithms, russac2020self}. However, their work is restricted to piece-wise stationary environments, characterized by the number $\Gamma_T$ of switches of the reward signal.
On the practical side, this drastically narrows down the non-stationary scenarios that can be efficiently addressed, as the measure $\Gamma_T$  can grossly overestimate the importance of the non-stationarity. In such case, any algorithm based on this measure will be sub-optimal and discard too fast previous data, quickly judged uninformative since the level of non-stationarity is expected to be high. This is typically the case in environments with many switches of small amplitude, characteristic of smooth drifts (e.g user-fatigue in recommender systems). On the theoretical side, this approach tells us little about the difficulties and challenges brought by the non-stationarity, as it relies on the fact that far enough from a switch, the environment is stationary. 
On the contrary, the variation-budget metric $B_T$ introduced and discussed in \citet[Section 2]{besbes2014stochastic}, allows for much finer considerations. It stands as a powerful characterization of the non-stationarity, measuring the number of switches and their amplitude \emph{jointly}. As a result, it can efficiently cover different scenarios, from drifting to piece-wise stationary environments. An adequate treatment of \glb{s} under this superior metric is therefore a crucial missing piece, and requires a sensibly different analysis and an appropriate algorithmic design.

\ifbulletpoint
{\color{red}
\begin{itemize}
		\item Existing approaches are wrong and it is serious:
			\begin{itemize}
				\item Implicitly uses the fact that $\hat\theta_t\in\Theta$ (foot note about the restart paper which define $\cm=0$ ?) 
			\end{itemize}
\end{itemize}}
\else
\fi
\paragraph{Parameter-drift and GLBs: flaws of previous approaches.} Most of the existing non-stationary \lb{} algorithms address the parameter-drift setting and their extension to \glb{s} was at first considered as relatively straight-forward \citep{cheung2019hedging,zhao2020simple}. Unfortunately, existing analyses suffer from important caveats because they overlook a crucial feature of \glb{s}.  Following \citet{filippi2010parametric}, they rely on a linearization of the reward function around $\hat\theta_t$. Naturally, the linear approximation must accurately describe the \emph{effective} behavior of the reward signal (characterized by the ground-truth $\ts{t}$). From Assumption~\ref{ass:bounded_reward}, this translates in the structural constraint $\hat\theta_t \in \Theta$, which is implicitly assumed to hold in previous attempts. Unfortunately, there exists no proof guaranteeing that $\hat\theta_t \in \Theta$ could hold. Even worse, existing deviation bounds \cite[Theorem 1]{abbasi2011improved} rather suggest that in some directions, \emph{even in the stationary case}, $\hat\theta_t$ can grow to be  $\sqrt{\log(t)}$ far from $\Theta$! The situation is even worse under non-stationarity since, as we shall see, $\hat\theta_t$ can be $B_t$ far from $\Theta$. This flaw in the analysis is critical and cannot be easily fixed without severely degrading the regret guarantee. When $\hat\theta_t\notin\Theta$, this impacts the ratio $R_\mu$ which captures the degree of non-linearity of the inverse link function. For the highly non-linear logistic function, easy computations show that $\Rm\geq e^{SL}$. If we were to inflate the radius of the admissible set $\Theta$ from $S$ to $S+\delta_S$ (so that it contains $\hat\theta_t$), the estimated non-linearity of the reward function would be even stronger and $R_\mu$ would be multiplied by a factor $e
^{L\delta_S}$! Because the regret bound scales linearly with $\Rm$, this exponential growth would lead to prohibitively deficient performance guarantees.

\begin{rem}
	The fact that $\hat\theta_t$ can leave the admissible set $\Theta$ is not merely a theoretical construction inherited from potentially loose deviation bounds. As highlighted in Figure~\ref{fig:projection}, we can see in our numerical simulations that this often happens in practice when the environment is non-stationary.
\end{rem}

\ifbulletpoint
{\color{red}
\begin{itemize}
	\item This is well known: Filippi maps $\hat\theta_t$ back to the feasible set $\Theta$, and so does Russac in the abrupt change case
\end{itemize}
}
\else
\fi

\subsection{Non-stationary GLBs: challenges}
\label{subsec:blabla}
In their seminal work, \citet{filippi2010parametric} countered the aforementioned difficulty by introducing a \emph{projection} step, mapping $\hat\theta_t$ back to an admissible parameter $\tilde\theta_t \in \Theta$. Formally, they compute:
					\begin{align}
						\tilde\theta_t = \argmin_{\theta\in\Theta} \left\lVert g_t(\theta)-g_t(\hat\theta_t)\right\rVert_{\mathbf{V^{-1}_t}}\tag{\textbf{P0}}
						\label{eq:filippi_proj}
					\end{align}
and use $\tilde\theta_t $ to predict the performance of the available actions. The projection step \eqref{eq:filippi_proj} essentially incorporates the prior knowledge $\theta_\star \in \Theta$ (Assumption~\ref{ass:bounded_reward}) without degrading the learning guarantees of the maximum likelihood estimator. This strategy was also leveraged by \citet{russac2020algorithms}, which was made possible thanks to their piece-wise stationarity assumption. 

The situation is different in our setting, as the parameter-drift framework allows the sequence $\{\ts{t}\}$ to change \emph{at every round}. This introduces \textbf{(1)} the need to characterize two phenomenons of different nature	 that we will designate as \emph{learning} and \emph{tracking}. The former (learning) is linked to the deviation of the maximum-likelihood estimator $\hat\theta_t$ from its noiseless counterpart $\bar\theta_t$ (the estimator that one would have obtained if one could have averaged  an infinite number of realization of the trajectory). The later (tracking) measures the deviation of $\bar\theta_t$ from the current $\ts{t}$, due to an incompressible error inherited from the drifting nature of the sequence $\{\ts{s}\}_{s=1}
^{t}$. The learning and tracking mechanisms are both sources of deviation of $\hat\theta_t$ away from $\Theta$, each under a different metric. This leads to \textbf{(2)} a tension in the design of the projection as this requires to incorporate the knowledge $\{\ts{t}\}\in\Theta$, without degrading neither the learning nor the tracking guarantees. This rules out the projection step \eqref{eq:filippi_proj}, oblivious to the tracking aspect of the problem and which needs to be generalized to adapt to the two sources of deviation (i.e learning and tracking).

\section{Algorithm and regret bound}
\label{sec:contributions}
\subsection{Algorithm}

\ifbulletpoint
{\color{red}
\begin{itemize}
		\item two steps in the design: appropriate parameter for prediction and appropriate bonus. define confidence set here.
\end{itemize}
}
\else
\fi

This section is dedicated to the description of the design of our new algorithm \ouralgo. It operates in two steps: \textbf{(Step 1)} the computation of an appropriate admissible parameter $\tilde\theta_t\in\Theta$ (to be used for predicting the rewards associated with the actions $x\in\mathcal{X}_t$ available at round $t$) and \textbf{(Step 2)} the construction of a suitable exploration bonus to compensate for prediction errors. 

The first step builds on the following set, linked to the deviation incurred through the learning process:
		\begin{align}	
		\label{eq:generic.conf.set}
					\mcal{E}_t^\delta(\theta) \defeq \left\{\theta'\in\mbb{R}^d \text{ s.t } \Big\lVert g_t(\theta') - g_t(\theta)\Big\rVert_{\mbold{\widetilde{V}_t^{-1}}}\leq \beta_t(\delta) \right\},
			\end{align}
		where  $\beta_t(\delta)$ is a slowly-increasing function of time (to be defined later) and $\delta\in(0,1]$.

\begin{figure}[t]
\centering
\scalebox{1.1}{
    \begin{tikzpicture}
     \begin{axis}[
     	axis equal,
        axis x line=none,
        axis y line=none,
        xtick=\empty,
        ytick=\empty,
        scaled ticks=false,
        xmin=-20,
        xmax=21,
        ymin=-16,
        ymax=15,
        xlabel=,
        ylabel=,
   ]
   \def\radiusT{8}

   \fill[pattern=north west lines, pattern color=blue, opacity=0.5] 
  (axis cs:-10,0) circle [radius=\radiusT];
  \draw[color=blue, line width=1.5] (axis cs:-10,0) circle [radius=\radiusT];
   \node at (axis cs:-10-\radiusT,+\radiusT) {$\color{blue}\Theta$};
   
   \fill[color=black, line width=1.2, pattern= horizontal lines, opacity=0.2] plot [smooth cycle, tension=1] coordinates {(axis cs: 12,8) (axis cs: 10,1) (axis cs: 14,-8) (axis cs: 18,0)} node at (axis cs:12,14) {};
   \draw[color=black, line width=1.2] plot [smooth cycle, tension=1] coordinates {(axis cs: 12,8) (axis cs: 10,1) (axis cs: 14,-8) (axis cs: 18,0)} node at (axis cs:12,14) {};
    \node at (axis cs:18,8) {\scriptsize$\pmb{\mathcal{E}_t^{\delta}(\hat{\theta}_t)}$};
     \draw[color=black, line width=1.2, draw opacity=0.4, pattern=crosshatch dots, opacity=0.1] plot [smooth cycle, tension=1]  coordinates {(axis cs: -2,7) (axis cs: -2,1) (axis cs: 2,-9) (axis cs: 3,1)} node at (axis cs:12,14) {};
   \node at (axis cs:6,-9) {\scriptsize$\color{gray}\pmb{\mathcal{E}_t^{\delta}(\theta_t^p)}$};
   
   \node at (axis cs:15.5,1) {\footnotesize$\hat{\theta}_t$};
   \draw[fill=black, color=black] (axis cs: 14,0) circle [radius=0.3];
     \node at (axis cs:14,-5) {\footnotesize$\bar{\theta}_t$};
     \draw[fill=black, color=black]  (axis cs:12.5,-6) circle [radius=0.3];
   \node[align=right] at (axis cs:-4.5,-7.5) {\footnotesize$\theta_{\star}^{t}$};
   \draw[fill=black, line width=1.5] (axis cs:-6,-6.9) circle [radius=0.3];
   \node[align=right] at (axis cs:2,-1) {\footnotesize $\theta_t^{p}$};
   \draw[fill=black, color=black, draw opacity=0.5] (axis cs: 1,-2) circle [radius=0.3];
   \draw[fill=red, color=red, draw opacity=0.5] (axis cs: -2., 0) circle [radius=0.4];
   \node[color=red] at (axis cs: -0.5,1) {\color{red}\footnotesize$\tilde{\theta}_t$};

   \draw[<->] (axis cs:14,12) -- (axis cs: 1,12);
   \draw[dotted] (axis cs:1,-1) -- (axis cs: 1,12);
   \draw[dotted] (axis cs: 14,0) -- (axis cs: 14,12);
   \node[align=right, rotate=0] at (axis cs: 8,14) {$\pmb{\propto\! B_t}$};
   
   \draw[dotted] (axis cs: -6, -6.9) -- (axis cs:-6,-12);
   \draw[dotted] (axis cs: 12.5, -6) -- (axis cs:12.5,-12);
   \draw[<->] (axis cs: -6,-12) -- (axis cs: 12.5, -12);
   \node[align=right, rotate=0] at (axis cs: 3.25,-14) {$\pmb{\propto \! B_t}$};
    \end{axis}
    \end{tikzpicture}}
    \vspace{-10pt}
    \caption{Illustration of the different parameters of interest. As stated by Lemma~\ref{lemma:bound_first_drift} and Lemma~\ref{lemma:bound_second_drift}, the deviations ($\theta_t^p\leftrightarrow \hat\theta_t$) and ($\bar\theta_t\leftrightarrow\ts{t}$) are linked to the parameter-drift $B_t$. On the other hand, the deviations  ($\hat\theta_t\leftrightarrow \bar\theta_t$) and  ($\tilde\theta_t\leftrightarrow \theta^p_t$) are characterized by the stochastic nature of the problem.}
    \label{fig:lemmaseconddriftillustration}
\end{figure}
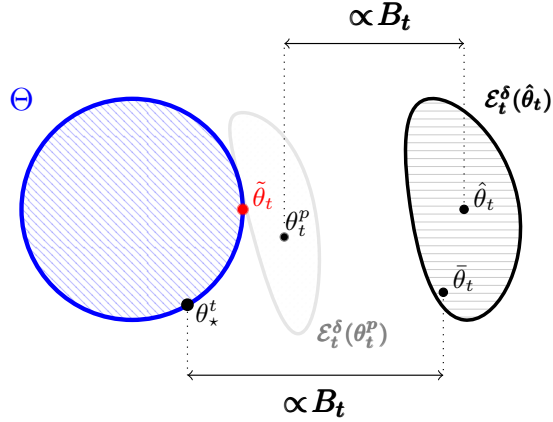

\ifbulletpoint
{\color{red}
\begin{itemize}
		\item introduction of $\theta_t^p$ and $\tilde\theta_t$ with motivation (informal)
\end{itemize}
}
\else
\fi
\textbf{Step 1.}  We start by identifying an intermediary parameter $\theta_t^p$, solution of the following constrained optimization program (ties can be broken arbitrarily):
			\begin{align}
				\theta_t^p \in \argmin_{\theta\in\mbb{R}^d}\left\{ \left\lVert g_t(\theta)-g_t(\hat\theta_t) \right\rVert_{\mathbf{V^{-2}_t}} \text{ s.t } \Theta\cap\mcal{E}_t^\delta(\theta)\neq \emptyset \right\}.\tag{\textbf{P1}}
				\label{opt:program_0}
			\end{align}
			The optimization program~\eqref{opt:program_0} is well-posed as it consists in minimizing a smooth function over a non-empty compact set\footnote{Notice that $\{\theta \text{ s.t } \Theta\cap\mcal{E}_t^\delta(\theta) \neq \emptyset \}$ always contains $0_d$, while the compactness is inherited from $\Theta$.}. Once $\theta_t^p$ is computed, the algorithm simply chooses any parameter $\tilde\theta_t\in\Theta\cap\mcal{E}_t^\delta(\theta_t^p)$. An efficient procedure to find such a parameter is detailed in Section~\ref{subsec:proj}. The different parameters of interest for \ouralgo{} are illustrated in Figure~\ref{fig:lemmaseconddriftillustration}.
\ifbulletpoint
{\color{red}
\begin{itemize}
		\item add a remark to emphasize the difference in projection with filippi
\end{itemize}
}
\else
\fi

\begin{rem}
		Notice the difference with the projection step used in the stationary case. In our case it is possible that $\mcal{E}_t^\delta(\hat\theta_t)$ (which is the confidence set centered at $\hat\theta_t$) does not intersect the admissible set $\Theta$. Our strategy for finding $\tilde\theta_t$ is then to compute an appropriate \textbf{vibration} $\mcal{E}_t^\delta(\theta_t^p)$ of $\mcal{E}_t^\delta(\hat\theta_t)$ which does intersect $\Theta$, while minimizing the deviation between $\theta_t^p$ and $\hat\theta_t$ according to a metric related to the tracking error (through the map $g_t$ and the squared inverse of the design matrix).
\end{rem}

\ifbulletpoint
{\color{red}
\begin{itemize}
		\item define the bonus
\end{itemize}
}
\else
\fi
\textbf{Step 2.} 
The exploration bonus at round $t$ for a given arm $x\in\mcal{X}_t$ is defined as $b_t(x)= 2\Rm\beta_t(\delta)\lVert x\rVert_{\mathbf{V^{-1}_t}}$, where $\delta\in(0,1]$ and:
		\begin{align*}
		\beta_t(\delta) = \sqrt{\lambda}\cm S + \sigma\sqrt{2\log(1/\delta)+d\log\left(1+\frac{L^2(1-\gamma^{2t})}{\lambda d(1-\gamma^2)}\right)}\; .
		\end{align*}
\ifbulletpoint
{\color{red}
\begin{itemize}
		\item and algorithm strategy
\end{itemize}
}
\else
\fi
\ouralgo{} then follows an optimistic strategy, boosting the predicted reward associated with $\tilde{\theta}_t$ by $b_t$ and plays $x_t \in\argmax_{x\in\mathcal{X}_t} \mu(\langle x,\tilde\theta_t\rangle) + b_t(x)$. The pseudo-code is summarized in Algorithm~\ref{alg:algo}.

\begin{algorithm}[!ht]
	\caption{\ouralgo}
	\label{alg:algo}
	\begin{algorithmic}
		\STATE {\bfseries Input.} regularization $\lambda$, confidence $\delta$, inverse link function $\mu$, weight $\gamma$, constants $S,L$ and $\sigma$.
		\STATE {\bfseries Initialization.} Compute $R_\mu$, let $\mathbf{V_1}\leftarrow \lambda\mathbf{I}_d$ and $\hat\theta_1 \leftarrow 0_d$.
		\FOR{$t\geq 1$}
		\STATE Find $\theta_t^p$ by solving \eqref{opt:program_0} and select $\tilde\theta_t\in\Theta \cap \mcal{E}_t^\delta(\theta_t^p)$.
		\STATE Play $x_t \leftarrow  \argmax_{x\in\mcal{X}_t} \mu(\langle x,\tilde\theta_t\rangle) +2\Rm\beta_t(\delta)\lVert x\rVert_{\mathbf{V^{-1}_t}}$.
		\STATE Observe reward $r_{t+1}$, update $\hat\theta_{t+1}$ by solving Equation~\eqref{eq:thetahatdef}.
		\STATE Update design matrix: $\mathbf{V_{t+1}}\leftarrow \gamma\mathbf{V}_t + x_tx_t^\transp + (1-\gamma)\lambda\mathbf{I_d}$.
		\ENDFOR
	\end{algorithmic}
\end{algorithm}

\ifbulletpoint
{\color{red}
\begin{itemize}
		\item one theorem with regret upper-bound
\end{itemize}
}
\else
\fi

\subsection{Regret bound}
We provide in Theorem~\ref{thm:regret_bound} a high-probability bound on the regret of \ouralgo.
	\begin{restatable}{thm}{thmregretbound}\label{thm:regret_bound}
			\ifappendix
			Let $\delta\in(0,1]$ and $D\in\mathbb{N}^*$. 
			Under Assumptions~\ref{ass:bounded_decision_set} -\ref{ass:bounded_reward}-\ref{ass:B_T_upper_bound} and 
			Assumption~\ref{ass:ortho}
			with probability at least $1-\delta$:
			\begin{align*}
				R_T \leq C_1 R_\mu \beta_T(\delta)  \sqrt{dT}\sqrt{T\log(1/\gamma) + \log\left(1+\frac{L^2(1-\gamma^T)}{\lambda d(1-\gamma)}\right) }+ C_2 R_\mu \frac{\gamma^D}{1-\gamma} T + C_3 R_ \mu D B_T
			\end{align*}
			Further, setting $\gamma=1-(B_T/(dT))^{2/3}$ ensures:
			\begin{align*}
				R_T = \bigO{\frac{\km}{\cm}d^{2/3}B_T^{1/3}T^{2/3}} \qquad \text{w.h.p}
			\end{align*}
			Under general arm-set geometry and Assumptions~\ref{ass:bounded_decision_set}-\ref{ass:bounded_reward}-\ref{ass:B_T_upper_bound}, with probability at least $1-\delta$
			\begin{align*}
			    R_T &\leq C_1 R_\mu \beta_T(\delta)  \sqrt{dT}\sqrt{T\log(1/\gamma) + \log\left(1+\frac{L^2(1-\gamma^T)}{\lambda d(1-\gamma)}\right) } 
			    \\ 
			    & + C_4 R_\mu \frac{\gamma^D}{1-\gamma} T
			    + C_5 \km R_\mu \frac{\gamma^D}{(1-\gamma)^{3/2}} T + C_6 k_\mu R_\mu \sqrt{\frac{d}{1-\gamma}} D B_T 
			    + C_7 k_\mu R_\mu \frac{\sqrt{d}}{1-\gamma} D B_T 
			\end{align*}
			Further, setting $\gamma = 1- \frac{B_T^{2/5}}{d^{1/5}T^{2/5}}$ ensures:
			\begin{align*}
				R_T = \bigO{\frac{\km}{\cm}d^{9/10} R_\mu B_T^{1/5}T^{4/5}} \qquad \text{w.h.p}
			\end{align*}
			\else
			Under Assumptions~\ref{ass:bounded_decision_set}-\ref{ass:bounded_reward}-\ref{ass:B_T_upper_bound} and \ref{ass:ortho}, setting $\gamma=1-(B_T/(dT))^{2/3}$ ensures that the regret of \ouralgo{} satisfies:
			\begin{align*}
				R_T = \bigO{\Rm d^{2/3}B_T^{1/3}T^{2/3}} \qquad \text{w.h.p}
			\end{align*}
			Under general arm-set geometry and Assumptions~\ref{ass:bounded_decision_set}-\ref{ass:bounded_reward}-\ref{ass:B_T_upper_bound}, setting $\gamma=1-(B_T/(\sqrt{d}T))^{2/5}$ ensures that the regret 
			of \ouralgo{} satisfies:
			\begin{align*}
				R_T = \bigO{\Rm d^{9/10} B_T^{1/5}T^{4/5}} \qquad \text{w.h.p}
			\end{align*}
			
			\fi
	\end{restatable}
\ifbulletpoint
{\color{red}
\begin{itemize}
		\item discussion about optimality and appearance of parameter $\Rm$ 
\end{itemize}}
\else
\fi
A few comments are in order. First, we note that as in the linear case, under Assumption~\ref{ass:ortho} the upper-bound on $R_T$ matches the asymptotic rates of the \lb{} lower-bound under parameter drift \cite[Theorem 1]{cheung2019hedging}. Without this assumption, the upper-bound suffers a small lag behind the LB rates, from $T^{3/4}$ to $T^{4/5}$. Second, one can notice the presence in the bound of the ratio $\Rm$, typical of the linearization approach performed to analyze \glb{s}. The bounds presented in Theorem~\ref{thm:regret_bound} are therefore quite natural and extends the work of \citet{filippi2010parametric} to non-stationary worlds. We emphasize that if the result seems unsurprising, it required a substantially different machinery, both for the design of the algorithm and its analysis. We highlight this last point in Section~\ref{sec:sketch_of_proof}, dedicated at providing a comprehensive sketch of proof for Theorem~\ref{thm:regret_bound}. The complete and detailed proof is deferred to Section~\ref{app:regret_bound} in the supplementary material.

\subsection{Solving the projection step}
\label{subsec:proj}
\ifbulletpoint
{\color{red}
\begin{itemize}
		\item  only $\tilde\theta_t$ is needed
\end{itemize}
}
\else
\fi
The optimization program \eqref{opt:program_0} and the subsequent search of a valid parameter $\tilde{\theta}_t$ can raise some legitimate concerns regarding the ease of practical implementation. Indeed, the feasible set of~\eqref{opt:program_0} is given by $\{\theta \hspace{1mm} \text{s.t.} \hspace{1mm} \Theta\cap\mcal{E}_t^\delta(\theta)\neq \emptyset\}$, where $\mcal{E}_t^\delta(\theta)$ is defined in~\eqref{eq:generic.conf.set}. Hence, the associated constraint is \emph{implicit} as it involves an additional \emph{non-convex} minimization program. As a result, it makes the constraint uneasy to manipulate and even hard to check. The same difficulty arises when searching for $\tilde{\theta}_t \in \Theta \cap \mcal{E}_t^\delta(\theta_t^p)$ where $\theta_t^p$ is a solution of~\eqref{opt:program_0}, due to the non-convexity of the set $\mcal{E}_t^\delta(\theta_t^p)$. 
The following proposition provides an alternative that avoids those difficulties.
\ifbulletpoint
{\color{red}
\begin{itemize}
		\item provide alternative optimization step
\end{itemize}
}
\else
\fi
\begin{restatable}{prop}{propfindingthetatilde}\label{prop:program1}
Let $\tilde{\theta}_t$ be such that:
	\begin{align}
		\begin{pmatrix} \tilde{\theta}_t\\ \eta_t^p  \end{pmatrix} \in\argmin_{\theta'\in\mbb{R}^d,\eta \in\mbb{R}^d}\left\{ \left\lVert g_t(\theta')+\beta_t(\delta)	\mathbf{\widetilde{V}^{1/2}_t}\eta-g_t(\hat\theta_t) \right\rVert_{\mathbf{V^{-2}_t}} \text{ s.t } \ltwo{\theta'}\leq S, \ltwo{\eta}\leq 1\right\}.\tag{\textbf{P2}}
		\label{opt:program_2}
	\end{align}
It exists $\theta_t^p$ solution of~\eqref{opt:program_0} such that $\tilde{\theta}_t \in \Theta \cap \mcal{E}_t^\delta(\theta_t^p)$.
\end{restatable}
\ifbulletpoint
{\color{red}
\begin{itemize}
		\item discuss relationship with filippi's projection
\end{itemize}}\else\fi
Proposition~\ref{prop:program1} shows that a valid $\tilde{\theta}_t$ can be found by solving~\eqref{opt:program_2}, bypassing the need to compute $\theta_t^p$. Essentially, the initial two-steps procedure to find $\tilde{\theta}_t$ (through the intermediary program~\eqref{opt:program_0}) is replaced by a single minimization program augmented with a slack variable $\eta$. The attentive reader may notice that~\eqref{opt:program_2} is now similar to~\eqref{eq:filippi_proj}, the projection step employed in~\citet{filippi2010parametric}. As a result, \ouralgo{} is comparable to the original algorithm \texttt{GLM-UCB} in terms of computational burden. The proof of Proposition~\ref{prop:program1} is given in Section~\ref{app:equivalent_min_proof} in the appendix.

\subsection{Online estimation of the variation-budget}
\label{sec:online_est}
\paragraph{Motivation.}
\ifbulletpoint{
\begin{itemize}
    \item \textcolor{red}{BVD needs to know an upper-bound on the variation-budget : same as before}
    \end{itemize}
} \else\fi

The attentive reader may notice that the minimax-optimality of \ouralgo{} is conditioned on the knowledge of an upper-bound $B_T$ for the true parameter-drift $B_{T,\star}$. Naturally, the tighter this upper-bound, the better the performance. Yet, whether such a knowledge is available in real-life problems is, to say the least, questionable. This issue is not specific to our approach but is shared with all non-stationary parametric bandit methods - see for instance \citep{cheung2019learning, zhao2020simple}.
\ifbulletpoint
{
    \begin{itemize}
        \item \textcolor{red}{In the linear case, and for a (window or restart) strategy, BOB allows to estimate online $B_T$ + regret}
    \end{itemize}
} \else\fi
For linear bandits, previous approaches circumvented this drawback with a Bandit-over-Bandit strategy \cite[Section 7]{cheung2019hedging}, where $B_{T,\star}$ is learned online by a \textit{master} algorithm. This guarantees sub-linear regret
without having the knowledge of $B_{T,\star}$.
\ifbulletpoint{
\begin{itemize}
    \item \textcolor{red}{As hinted earlier, the same arguments can be used to design a window version of BVD. Using BOB on top will probably yield similar guarantees.}
\end{itemize}}\else\fi
We however note that this technique was specialized for linear bandits and for the sliding-window strategy. As hinted in the introduction one could easily design a sliding-window approach of \ouralgo{} (using very similar arguments as the ones displayed in this paper) and extend the Bandit-over-Bandit of \cite{cheung2019hedging} to the GLB framework.
\ifbulletpoint{\begin{itemize}
\item \textcolor{red}{additional contribution: extends the BOB approach for weights in the GLM setting.}
\end{itemize}}\else\fi
Here, we follow a different path and introduce an equivalent method for the exponential-weighting strategy. To the best of our knowledge, this technique was missing in the non-stationary parametric bandit literature. It notably proves that the online learning of $B_{T,\star}$ can be efficiently performed under discounted strategies.
    
\paragraph{Bandit-over-Bandit for discounted strategies.}

\ifbulletpoint{\begin{itemize}
    \item \textcolor{red}{high level idea: we want to estimate $B_T^\star$. cover the interval (0,2TS) on a log-scale, find the related $\gamma$ and play EXP3 with each expert}\end{itemize}}\else\fi

For the sake of simplicity, we describe the Bandit-over-bandit approach adopted when Assumption~\ref{ass:ortho} holds. A similar reasoning holds in general but naturally yields different rates. Notice that naive bounding gives $B_{T,\star}\in(0,2ST]$. The main idea for learning $B_{T,\star}$ online
is to grid on a log-scale the interval $(0,2ST]$ with $N$ values $\{B_{T,j}\}_{j=1}^N$. We then create $N$ instances of \ouralgo{}, each set with a different discount factor: \begin{align*}
    \gamma_j = 1 - \left(\frac{B_{T,j}}{d T} \right)^{2/3} = 1-\frac{2^{j-1}}{2^{5/3}d^{2/3}TS^{2/3}} \; . 
\end{align*} These instances will be our \emph{experts}. We then deploy a \emph{master} algorithm - a version of $\EXP$ \citep{auer2002nonstochastic}, which acts repeatedly as follows: \textbf{1.} it chooses an expert $j$ (\emph{i.e} a new instance of \ouralgo{} with parameter $\gamma_j$) to interact with the environment during a time frame of length $H$ ($H$ is a positive integer). \textbf{2.} The master algorithm then observes the cumulative reward (aggregated on the time frame) of the expert $j$. We give the pseudo-algorithm of this procedure in Algorithm~\ref{alg:meta_informal}.
\renewcommand{\algorithmiccomment}[1]{\hfill\eqparbox{COMMENT}{\# #1}}

\begin{algorithm}
\caption{\ouralgoBOB{} (a more detailed version is deferred to Appendix~\ref{app:bob_algo}).}
\label{alg:meta_informal}
  \begin{algorithmic}
    \STATE{ \bfseries Input.} Length $H$, time
    horizon $T$, regularization $\lambda$, confidence $\delta$, inverse link function $\mu$, 
    constants $S,L$ and $\sigma$.
    \STATE {\bfseries Initialization.}
    Let $N\leftarrow\lceil2\log_2(2ST^{3/2})\rceil$ and $\mathcal{H}\leftarrow \{\gamma_j=1-\frac{2^{j-1}}{2^{5/3}d^{2/3}TS^{2/3}}\}_{j=1}^N $, initialize $\EXP$ with action set indexed by $\mcal{H}$.
    \FOR{$i= 1,\ldots, \lceil T/H\rceil$}
    \STATE{$j\leftarrow$ action selected by $\EXP$}. \\
    \STATE Initialize a sub-routine \ouralgo{} with parameter $\gamma_j$.\\
    \FOR{$t = 1, \dots, H$}
    \STATE{Play with \ouralgo{} with parameter $\gamma_j$, observe reward $r_{t+1}$.}
    \ENDFOR
    \STATE{Update $\EXP$ with reward $\sum_{t=1}^H r_{t+1}$.}
    \ENDFOR
  \end{algorithmic}
\end{algorithm} 

Informally, the idea is that $\EXP$ will learn to select the best performing $\gamma_j$ associated with the best estimate $B_{T,j}$ of $B_{T,\star}$. Intuitively, this should guarantee small regret as $\EXP$ will mostly play instances of \ouralgo{} which nearly capture the true magnitude of the non-stationarity. This intuition is made rigorous in Theorem~\ref{thm:regretmaster}, whose proof is deferred to Section~\ref{app:bob} in the appendix. 
    
\begin{restatable}{thm}{thmregretmaster}\label{thm:regretmaster}
    Under Assumptions~\ref{ass:bounded_decision_set}-\ref{ass:bounded_reward} and \ref{ass:ortho}, 
    the regret of \ouralgoBOB{} when setting 
    $H=\lfloor d \sqrt{T}\rfloor$ 
    satisfies:
    \begin{align*}
        \mathbb{E}[R_T] = \bigO{R_\mu d^{2/3}T^{2/3}\max\left( B_{T,\star},d^{-1/2}T^{1/4}\right)^{1/3}} \;.
    \end{align*}
    \end{restatable}

    \ifbulletpoint{\begin{itemize}
    \item \textcolor{red}{regret guarantee: same as BOB}\end{itemize}}\else\fi

Essentially, we obtain a regret bound which is identical to
the ones of the Bandit-over-Bandit algorithms of \cite{cheung2019hedging} and \cite{zhao2020simple}. The conclusions are therefore of similar nature: namely, when $B_{T,\star}  \geq d^{-1/2}T^{1/4}$ we obtain a minimax rate, \emph{without} knowing $B_{T,\star}$. Again, note here the presence of the problem-dependant constant $R_\mu$, inherited from the non-linear reward structure imposed in GLBs.

\section{Proof sketch}
\label{sec:sketch_of_proof}

In this section, we detail the key steps of the proof of Theorem~\ref{thm:regret_bound}. In particular, we shed light on the tension between the learning and tracking aspects of the problem and their role in the choice of the estimator $\tilde{\theta}_t$, through the use of an appropriate projection step. For simplicity we assume that Assumption~\ref{ass:ortho} holds, although the spirit of the proof is almost identical in the general case.

\paragraph{Learning versus tracking.}

\ifbulletpoint
{\color{red}
\begin{itemize}
		\item separation of learning versus tracking. introduction of $\bar\theta_t$ (via minimization of convex function and provide intuitive meaning).
\end{itemize}
}
\else
\fi
A crucial feature of non-stationary GLBs lies in the singular nature of the deviation of $\hat\theta_t$ from $\ts{t}$. This arises from two fundamentally different mechanisms: learning and tracking. We introduce the following estimator, which allows for a clean-cut distinction between the two phenomenons:
		\begin{align}
			\bar\theta_t \defeq \argmin_{\theta\in\mbb{R}^d}\left\{ \sum_{s=1}^{t-1}\gamma^{t-1-s}\left[ b(\langle x_s,\theta\rangle)-\mu\left(\langle x_s, \ts{s} \rangle\right)\langle x_s,\theta\rangle\right]+\frac{\lambda c_\mu}{2}\ltwo{\theta-\ts{t}}^2\right\}.
			\label{eq:def_theta_bar}
		\end{align}
The parameter $\bar\theta_t$ is the minimizer of a strictly convex and coercive function, thus  is well-defined and unique. Intuitively, $\bar\theta_t$ would be the estimator obtained under a perfect (e.g noiseless) observation of the reward\footnote{Note the difference between $\hat{\theta}_t$ and $\bar{\theta}_t$, where the rewards $r_{t+1}$ are replaced by their conditional expected values $\mu\left(\langle x_s, \ts{s} \rangle\right)$}. As a result, the deviation between $\hat\theta_t$ and $\bar\theta_t$ is solely due to the stochastic nature of the problem (\emph{learning}). On the other hand, the deviation between $\bar\theta_t$ and $\ts{t}$ is a consequence of the unpredictable changes of the sequence $\{\ts{s}\}_{s}$ (\emph{tracking}). The introduction of the reference point $\bar{\theta}_t$ allows us to characterize both deviations separately in Lemma~\ref{lemma:confidence_set} and Lemma~\ref{lemma:bound_first_drift}.
\ifbulletpoint
{\color{red}
\begin{itemize}
		\item  construction of confidence set (first technical lemma on concentration: use tools from Russac)
\end{itemize}
}
\else
\fi	
\begin{restatable}{lemm}{lemmaconfidenceset}[Learning]
	Let $\delta\in(0,1]$. With probability at least $1-\delta$:
	\begin{align*}
		\text{ for all } t\geq 1, \quad \bar\theta_t \in \mcal{E}_t^\delta(\hat\theta_t)=\left\{\theta\in\mbb{R}^d \text{ s.t } \left\lVert g_t(\theta) - g_t(\hat\theta_t)\right\rVert_{\mbold{\widetilde{V}_t^{-1}}}\leq \beta_t(\delta) \right\}. 
	\end{align*}
	\label{lemma:confidence_set}
\end{restatable}
Lemma~\ref{lemma:confidence_set} ensures that with high probability the set $\mcal{E}_t^\delta(\hat\theta_t)$ is a \emph{confidence set} for $\bar\theta_t$. A complete proof of this result is deferred to Section~\ref{app:confidence_set_app} in the supplementary material. 
\ifbulletpoint
{\color{red}
\begin{itemize}
		\item mention that $\bar\theta_t$ can be outside of $\Theta$. \underline{is this really useful for the discussion} 	
		\item characterization of $\ell_2$ norm between $\bar{\theta}_t$ and $\ts{t}$.
\end{itemize}
}
\else
\fi	
\begin{restatable}{lemm}{lemmaboundfirstdrift}[Tracking with orthogonal action sets]
	Let $D\in\mbb{N}^*$. The following holds:
	\begin{align*}
		\lVert g_t(\bar\theta_t) -g_t(\ts{t})\rVert_{\mathbf{V_t^{-2}}} \leq \frac{2\km L^2 S}{\lambda} \frac{\gamma^D}{1-\gamma} + \km\sum_{s=t-D}^{t-1} \ltwo{\ts{s}-\ts{s+1}}.
	\end{align*}
	\label{lemma:bound_first_drift}
\end{restatable}

Lemma~\ref{lemma:bound_first_drift} effectively links the deviation of $\bar\theta_t$ from $\ts{t}$ to the variation-budget $B_T$ through the drift $\sum_{s=t-D}^{t-1} \ltwo{\ts{s}-\ts{s+1}}$. The proof of this result borrows tools from \cite{russac2019weighted} and is deferred to Section~\ref{sec:prooflemma:bound_first_drift} in the supplementary material. The integer $D$ appearing in Lemma~\ref{lemma:bound_first_drift} is introduced for the sake of the analysis only. It allows to treat separately old and recent observations. We  provide its optimal value later in this section. 

\ifbulletpoint
{\color{red}
\begin{itemize}
		\item \underline{add a remark} to emphasize that a projection step as in fillipi was not possible because those terms are controlled under different norms
\end{itemize}
}
\else
\fi	

\begin{rem}
	Behind the statement  of Lemma~\ref{lemma:confidence_set} and Lemma~\ref{lemma:bound_first_drift} hides the main reason why the projection step of \citet{filippi2010parametric} needs to be generalized. Indeed, it appears that the deviations $(\hat\theta_t \leftrightarrow \bar\theta_t)$ and $(\bar\theta_t\leftrightarrow \ts{t})$ are controlled through different metrics ($\mathbf{\widetilde{V}^{-1}_t}$ and $\mathbf{V^{-2}_t}$, respectively). Projecting according to the first metric would corrupt the control of the second deviation, and conversely.
\end{rem}

\paragraph{Regret decomposition and prediction error.}

To bound the instantaneous regret at round $t$, we rely on the prediction error $\Delta_t$ defined as follows for any arm $x\in\mathcal{X}_t$:
\begin{align*}
		\Delta_t(x) \defeq \left\vert \mu\left(\langle x,\tilde\theta_t\rangle\right)-\mu\Big(\langle x,\ts{t}\rangle\Big)\right\vert.
\end{align*}

\ifbulletpoint
{\color{red}
\begin{itemize}
		\item regret decomposition with delta-pred for $\tilde{\theta}_t$
\end{itemize}
}
\else
\fi	
The next Lemma ties the cumulative pseudo-regret to the sum of prediction errors. This derivation is classical and the proof is deferred to Section~\ref{app:regret_decomposition} in the supplementary material. 
		\begin{restatable}{lemm}{lemmaregretdecomposition}
		\label{lemma:regret_decomposition}
			The following holds:
			\ifappendix
			\begin{align*}
				R_T \leq \frac{2\km}{\cm}\sum_{t=1}^{T} \beta_t(\delta)\left[\lVert x_t\rVert_{\mathbf{V_t^{-1}}}-\lVert x_\star^t\rVert_{\mathbf{V_t^{-1}}}\right] +\sum_{t=1}^{T}\left[\Delta_t(x_t)+\Delta_t(x_\star^t)\right].
			\end{align*}
			\else
			\begin{align*}
				R_T \leq 2\Rm\sum_{t=1}^{T} \beta_t(\delta)\left[\lVert x_t\rVert_{\mathbf{V_t^{-1}}}-\lVert x_\star^t\rVert_{\mathbf{V_t^{-1}}}\right] +\sum_{t=1}^{T}\left[\Delta_t(x_t)+\Delta_t(x_\star^t)\right].
			\end{align*}
			\fi
		\end{restatable}

\ifbulletpoint
{\color{red}
\begin{itemize}
		\item decompose delta pred, packing terms together
\end{itemize}
}
\else
\fi	
Thanks to Lemma~\ref{lemma:regret_decomposition} we are left to characterize the prediction error $\Delta_t(x)$ for any $x\in\mathcal{X}_t$. Following \citet {filippi2010parametric}, we rely on the mean-value theorem to ensure that it exists $\mathring{\theta}_t\in[\tilde\theta_t,\ts{t}]$   such that\footnote{Formally,   $\mathring{\theta}_t\in[\tilde\theta_t,\ts{t}]$ means that there
exists $v\in[0,1]$ such that $\mathring{\theta}_t = v \tilde\theta_t + (1-v) \ts{t}$.}:
		\begin{align}
				\Delta_t(x) \leq \km \left\langle x, \mathbf{H_t}(\mathring{\theta}_t) \left(g_t(\tilde\theta_t)-g_t(\ts{t})\right)\right\rangle,
			\label{eq:mvt_deltapred}
			\end{align}
			where $\mathbf{H_t(\theta)}\defeq \sum_{s=1}^{t-1}\dot\mu(\langle x_s,\theta \rangle)x_sx_s^\transp+\lambda\cm\mathbf{I_d}$. Since $\tilde\theta_t, \ts{t}\in\Theta$, we obtain by convexity that $\mathring{\theta}_t \in \Theta$ and we can use the lower bound $\mathbf{H_t}(\mathring{\theta}_t)\succeq \cm\mathbf{V_t}$. 	
			\begin{rem}
				In this last inequality resides the mistake that was made in previous extension of \citet{filippi2010parametric} to the non-stationary setting \citep{cheung2019hedging,zhao2020simple}. Indeed, if the prediction error is measured at $\hat\theta_t$, we are left with $\mathring{\theta}_t\in[\ts{t},\hat\theta_t]$, and $\mathring\theta_t$ can lie outside of the admissible set $\Theta$ (since $\hat\theta_t$ can). The lower-bound linking $\mathbf{H_t}(\mathring{\theta}_t)$  and $\mathbf{V_t}$ would therefore not hold. More precisely, and as detailed in Section~\ref{subsec:limitations}, when $\mathring\theta_t\in[\ts{t},\hat\theta_t]$ not much can be said on the link between $\mbold{H_t}(\mathring\theta_t)$ and $\mbold{V_t}$ without severely degrading the final regret guarantees.
			\end{rem}
			Adding and removing $g_t(\hat\theta_t)+g_t(\theta_t^p)+g_t(\bar\theta_t)$ inside the inner-product in Equation~\eqref{eq:mvt_deltapred}, followed by easy manipulations yields:
		\begin{align*}
				\Delta_t(x) \leq &\underbrace{\Rm\left\lVert  x\right\rVert_{\mathbf{V^{-1}_t}}\left(\left\lVert g_t(\tilde\theta_t) -g_t(\theta_t^p)\right\rVert_{\mathbf{\widetilde{V}^{-1}_t}} + \left\lVert g_t(\bar\theta_t) -g_t(\hat\theta_t)\right\rVert_{\mathbf{\widetilde{V}^{-1}_t}}\right)}_{\defeq \Delta_t^{\text{learn}}(x)} \\ 
				&+ \underbrace{\Rm\ltwo{x}\left(\left\lVert g_t(\theta_t^p) -g_t(\hat\theta_t)\right\rVert_{\mathbf{V^{-2}_t}} + \left\lVert g_t(\bar\theta_t) -g_t(\ts{t})\right\rVert_{\mathbf{V^{-2}_t}}\right)}_{^{\defeq \Delta_t^{\text{track}}(x)}}.
		\end{align*}

\paragraph{Leveraging the projection step}
\ifbulletpoint
{\color{red}
\begin{itemize}
		\item first bound on the learning
\end{itemize}
}
\else
\fi	
We can now bound the terms $\Delta_t^{\text{learn}}(x)$ and $ \Delta_t^{\text{track}}(x)$ separately. Lemma~\ref{lemma:confidence_set} along with the design $\tilde\theta_t\in\mcal{E}_t^\delta(\theta_t^p)$ leads to:
\begin{align}
		\Delta_t^{\text{learn}}(x) \leq 2\Rm\left\lVert  x\right\rVert_{\mathbf{V^{-1}_t}}\beta_t(\delta) \quad \text{w.h.p}
	    \label{eq:main_delta_pred}
\end{align}
\ifbulletpoint
{\color{red}
\begin{itemize}
		\item  second bound on tracking. technical lemma: prove tracking relationship between $\theta_t^p$ and $\hat{\theta}_t$
\end{itemize}}
\else
\fi
The first term in $\Delta_t^{\text{track}}(x)$ is kept under control by the specific design of the projection step \eqref{opt:program_0}. This is formalized in the following Lemma, whose
proof is deferred to Section~\ref{sec:proof_lemma:bound_second_drift} in the appendix.
\begin{restatable}{lemm}{lemmaboundseconddrift}\label{lemma:bound_second_drift}
			Under the event $\{\bar\theta_t\in\mcal{E}_t^\delta(\hat\theta_t)\}$ the following holds:
			\begin{align*}
				\lVert g_t(\theta_t^p) -g_t(\hat\theta_t)\rVert_{\mathbf{V_t^{-2}}} \leq \lVert g_t(\bar\theta_t) -g_t(\ts{t})\rVert_{\mathbf{V_t^{-2}}}\;.
			\end{align*}
\end{restatable}
As a result, bounding $\Delta_t^{\text{track}}(x)$ reduces to bounding $\lVert g_t(\bar\theta_t)-g_t(\ts{t})\rVert_{\mathbf{V^{-2}_t}}$. Combined with Lemma~\ref{lemma:bound_first_drift}, this result states that the deviation between $\theta_t^p$ and $\hat\theta_t$ is characterized by $B_t$, the parameter-drift up to round $t$, as illustrated in Figure~\ref{fig:lemmaseconddriftillustration}. This leads to:
	\begin{align}
	\label{eq:main.delta.track}
		\Delta_t^{\text{track}}(x) \leq 2\Rm\left\lVert  x\right\rVert_2 \left(  \frac{2\km L^2 S}{\lambda} \frac{\gamma^D}{1-\gamma} + \km\sum_{s=t-D}^{t-1} \ltwo{\ts{s}-\ts{s+1}} \right) \quad \text{w.h.p}
	\end{align}

\paragraph{Putting everything together.} 
\ifbulletpoint
{\color{red}
\begin{itemize}
		\item final bound prediction error, plug in regret
\end{itemize}
}
\else
\fi

Combining Equations~\eqref{eq:main_delta_pred} and \eqref{eq:main.delta.track} with Lemma~\ref{lemma:regret_decomposition} and the Elliptical Lemma (Lemma~\ref{lemma:ellipticalpotential} in the supplementary material) yields:
		\begin{align*}
			R_T \leq C_1\Rm dT\log(1/\gamma) + C_2\Rm\gamma^DT/(1-\gamma) + C_3\Rm DB_T\qquad \text{w.h.p}
		\end{align*}
where the constants $C_{1}$, $C_{2}$ and $C_{3}$ hide $\log(T)$ multiplicative dependencies. A detailed proof of this result is deferred to Section~\ref{app:complete_regret_bound} in the supplementary material. Setting the hyper-parameters $D=\log(T)/(1-\gamma)$ and $\gamma=1-(\frac{B_T}{dT})^{2/3}$ concludes the proof of Theorem~\ref{thm:regret_bound}.

\section{Experiments}
\label{sec:exps}
We illustrate in Figure~\ref{fig:exps} the behavior and performance of \ouralgo{} with numerical simulations in a two-dimensional non-stationary logistic environment. Formally, we let $r_{t+1}\sim \text{Bernoulli}(\mu(\langle x_t, \ts{t}\rangle))$ where $\mu(z)=(1+e^{-z})^{-1}$ is the logistic function. The sequence $\{\ts{t} \}_{t \geq 1}$ evolves as follows: we let $\ts{t}= (0,1)$ for $t\in[1,T/3]$. Between $t=T/3$ and $t=2T/3$ we smoothly rotate $\ts{t}$ from $(0,1)$ to $(1,0)$. Finally  we let $\ts{t}= (0,1)$ for $t\in[2T/3,T]$. A thorough description of the experimental setting can be found in Appendix~\ref{app:exps}. We compare in Figure~\ref{fig:regrets} the four following algorithms: \texttt{OFUL} \citep{abbasi2011improved} (stationary, here mispecified), \texttt{GLM-UCB} \citep{filippi2010parametric} (stationary, here well-specified), \texttt{D-LinUCB} \citep{russac2019weighted} (an exponentially weighted LB algorithm, non-stationary but here mispecified) and \ouralgo{} (non-stationary, well-specified). For \texttt{D-LinUCB} and \ouralgo{} we use the value of $\gamma$ recommended by the asymptotic analysis. This figure highlights the necessity to employ algorithms that are well-specified; both \texttt{GLM-UCB} and \ouralgo{} outperform their linear counterparts (\texttt{OFUL} and \texttt{D-LinUCB}, respectively). Note that an appropriate treatment of non-stationarity is also crucial to obtain small regret as for the considered horizon the two best performing algorithms are \texttt{D-LinUCB} and \ouralgo. The latter being well-specified and resilient to non-stationary, it naturally performs best. In Figure~\ref{fig:projection} we highlight the fact that the projection step is necessary as, in this non-stationary setting, $\hat\theta_t$ regularly leaves the admissible set $\Theta$.

\begin{figure}[t]
\floatconts
{fig:exps}
{\caption{Numerical simulations in a non-stationary logistic setting. For the first figure, results are average over 50 independent runs and shaded areas represent one standard-deviation variation.}}
{
\subfigure[Regret bounds of different stochastic bandit algorithms under parameter-drift. The grey region indicates a smooth drift of $\ts{t}$.]{%
\centering
{\includegraphics[width=0.45\linewidth]{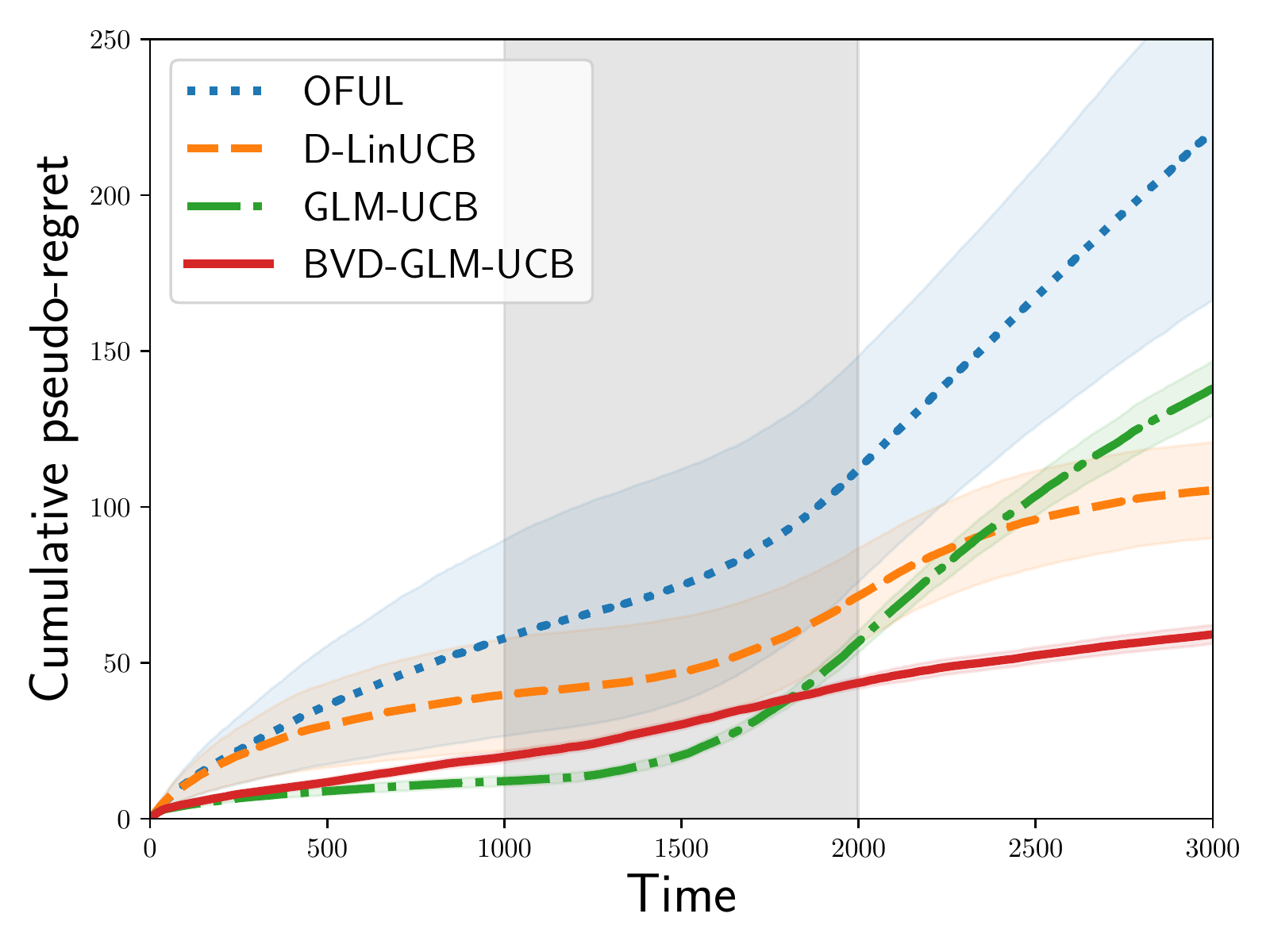}}
\label{fig:regrets}
}\quad
\subfigure[Evolution of the parameters of interest ($\ts{t},\hat\theta_t,\tilde\theta_t$) for \ouralgo. Note that in this non-stationary setting $\hat\theta_t\notin\Theta$ is frequent.]{%
\centering
{\includegraphics[width=0.45\linewidth]{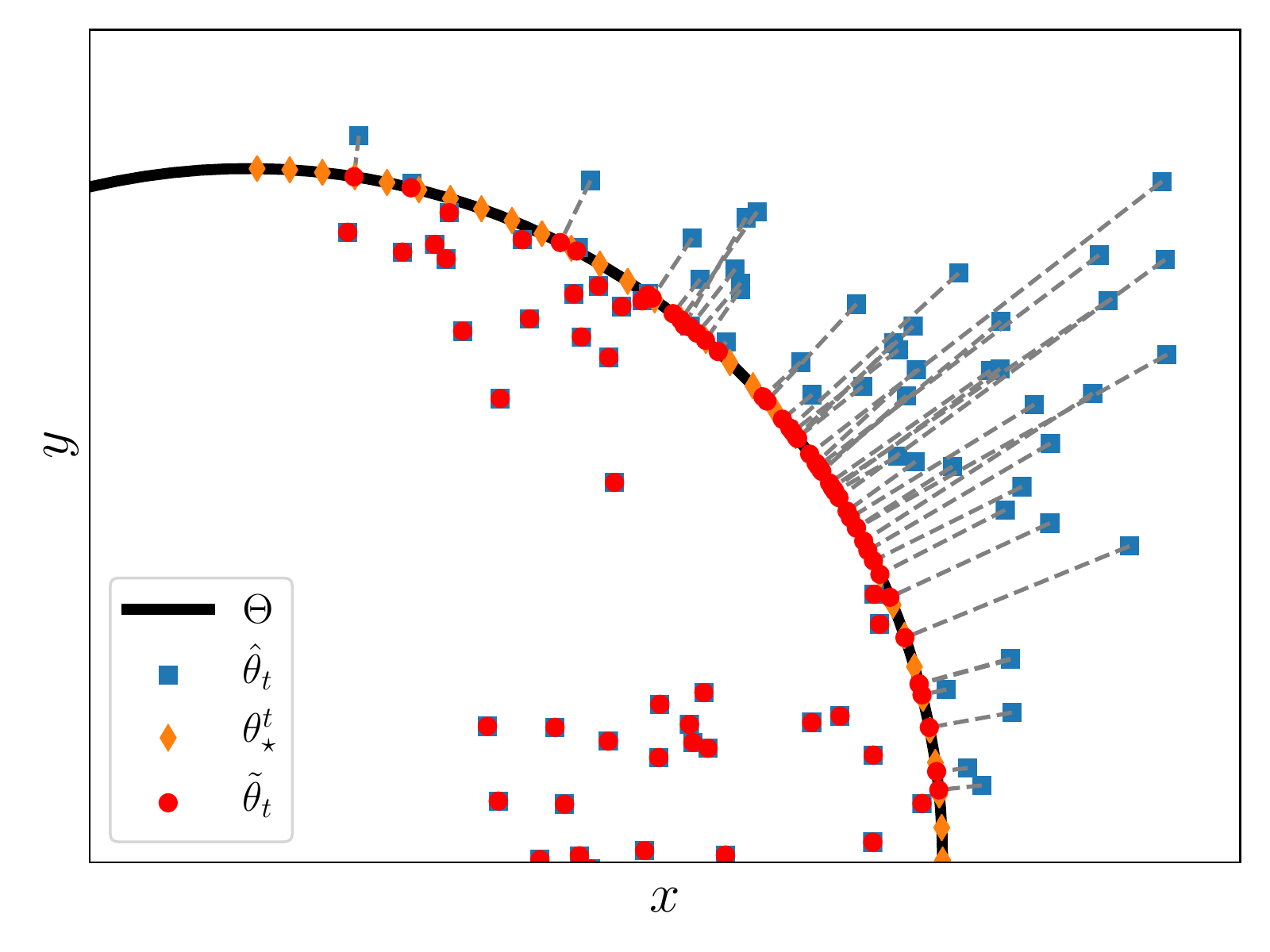}}
\label{fig:projection}
}
}
\end{figure}

\section*{Conclusion and future work}
We highlight in this paper a central difficulty in the theoretical treatment of non-stationary GLBs, overlooked in existing approaches and intimately linked to the non-linear nature of the reward function. To overcome this difficulty, we introduce a generalization of the projection step from \citep{filippi2010parametric}, which allows to simultaneously \emph{track} the non-stationary ground-truth while preserving the \emph{learning} guarantees of weighted maximum-likelihood strategies. This novel algorithmic design along with a careful analysis proves that an order-optimal (w.r.t $d$, $T$ and $B_T$) regret-bound can be achieved for GLBs under parameter-drift, although up to a rather restrictive assumption on the arm set's geometry. The nature of the minimax-rates in the general case is open, as in both LB \citep[\emph{c.f}][]{touati2020efficient} and GLB setting (this work) we observe a mismatch between existing upper-bounds and the lower-bound of \citep{cheung2019learning}. 

We underlined in Section~\ref{subsec:limitations} the problematic scaling of the problem-dependent constant $R_\mu$. Consequent research efforts have recently been deployed to reduce its impact on regret-bounds, both in the stationary \citep{faury2020improved,abeille2020instance,jun2020improved} and piece-wise stationary \citep{russac2020self} settings. What is the optimal dependency w.r.t $R_\mu$ in the more general parameter-drift setting, and how it can be achieved are exciting open questions that we here leave for future work.

\paragraph{Acknowledgements.} LF thanks Olivier Fercoq for his helpful observations regarding the derivation of the simplified projection step. YR thanks Arnaud Buisson and Jianjun Yuan for fruitful discussions on the BOB framework and its extension using discount factors.

\clearpage
\bibliography{bib}

\newpage
\appendix
\appendixtrue

\section*{Organization of the appendix}
The appendix is organized as follows:
\begin{itemize}
	\item In Section~\ref{app:concentration} we provide some concentration results, along with a bound on the prediction error $\Delta_t$ inherited from the design of the projection step.
	\item In Section~\ref{app:regret_bound} we link the prediction error $\Delta_t$ to the regret $R_T$ of \ouralgo{}. We then proceed to prove the bound on $R_T$ announced in Theorem~\ref{thm:regret_bound}.
	\item In Section~\ref{app:equivalent_min_proof} we provide a proof for the equivalence of the optimization programs \eqref{opt:program_0} (along with the computation of $\tilde\theta_t$) and \eqref{opt:program_2}.
	\item Section~\ref{app:useful} contains some secondary lemmas needed for the analysis, such as a version of the Elliptical Lemma for weighted matrices.
	\item In Section~\ref{app:bob} we provide a proof for the regret upper-bound of \texttt{BOB-BVD-GLM-UCB} claimed in Theorem~\ref{thm:regretmaster}.
	\item Finally, in Section~\ref{app:exps} we provide some details on our numerical simulations.
\end{itemize}

\section{Concentration and predictions bound}
\label{app:concentration}

\subsection{Confidence sets}
\lemmaconfidenceset*
\label{app:confidence_set_app}
\begin{proof}
		Recall that:
		\begin{align*}	
			\mcal{E}_t^\delta(\hat\theta_t) = \left\{\theta\in\mbb{R}^d \text{ s.t } \lVert g_t(\theta) - g_t(\hat\theta_t)\rVert_{\mbold{\widetilde{V}_t^{-1}}}\leq \beta_t(\delta) \right\} \;,
		\end{align*}
		where
		\begin{align*}
		\beta_t(\delta) = \sqrt{\lambda}\cm S + \sigma\sqrt{2\log(1/\delta)+d\log\left(1+\frac{L^2(1-\gamma^{2t})}{\lambda d(1-\gamma^2)}\right)}\;.
		\end{align*} 
		Also, from the definition of $\bar\theta_t$ in Equation~\eqref{eq:def_theta_bar}, by setting to 0 the differential of the convex objective minimized by $\bar\theta_t$ we obtain that:
		\begin{align}
			g_t(\bar\theta_t) = \sum_{s=1}^{t-1} \gamma^{t-1-s}\mu\left(\langle \ts{s},x_s\rangle\right)x_s + \lambda\cm\ts{t}\; .
			\label{eq:theta_bar_gt}
		\end{align}
		Further, for all $s\geq 1$, define 
		\begin{align}
			\epsilon_{s+1} = r_{s+1}-\mu(\langle \ts{s}, x_s\rangle)\;.
			\label{eq:eta_def}
		\end{align}
		Let $\tilde{\mcal{F}}_s=\sigma(x_1, r_2, .., x_{s-1},r_s,x_s)$, which compared to $\mcal{F}_s$ includes the arm $x_s$. Note that:
		\begin{equation*}
		\left\{
		\begin{aligned}
			\mbb{E}&\left[ \epsilon_{s+1}\middle| \tilde{\mcal{F}}_s \right] = 0   &\text{(Equation~\eqref{eq:reward_def})}\\
			-\mu(\langle \ts{s}, x_s\rangle) \leq &\epsilon_{s+1} \leq 2\sigma+\mu(\langle \ts{s}, x_s\rangle) \quad \text{a.s}  &\text{(Assumption~\ref{ass:bounded_reward})}
		\end{aligned}\right.
		\end{equation*}
		Therefore $\epsilon_{s+1}$ is $\sigma$-subGaussian conditionally on $\tilde{\mcal{F}}_s$.  Furthermore, by optimality of $\hat\theta_t$, differentiating the objective function in Equation~\eqref{eq:thetahatdef} yields:
		\begin{align}
			&\sum_{s=1}^{t-1} \gamma^{t-1-s}\left[\mu(\langle \hat\theta_t, x_s\rangle)-r_{s+1}\right]x_s + \lambda\cm\hat\theta_t = 0 &\notag\\
			&\pmb{\Leftrightarrow} g_t(\hat\theta_t) = \sum_{s=1}^{t-1}\gamma^{t-1-s}\mu(\langle \ts{s}, x_s\rangle)x_s + \sum_{s=1}^{t-1}\gamma^{t-1-s}\epsilon_{s+1}x_s &\text{(Equation~\eqref{eq:eta_def})}\notag\\
			& \pmb{\Leftrightarrow} g_t(\hat\theta_t) = g_t(\bar\theta_t)+\sum_{s=1}^{t-1}\gamma^{t-1-s}\epsilon_{s+1}x_s -\lambda\cm\ts{t}&\text{(Equation~\eqref{eq:theta_bar_gt})} \label{eq:noideafornamee}\\
			&\pmb{\Leftrightarrow}  \lVert g_t(\bar\theta_t) - g_t(\hat\theta_t)\rVert_{\mbold{\widetilde{V}_t^{-1}}}= \left\lVert\sum_{s=1}^{t-1}\gamma^{t-1-s}\epsilon_{s+1}x_s -\lambda\cm\ts{t} \right\rVert_{\mbold{\widetilde{V}_t^{-1}}}\notag \;.
		\end{align} %

Therefore since $\ts{t}\in\Theta$ and $\mbold{\widetilde{V}_t} \succeq \lambda \mbold{I_d}$ we obtain:
\begin{align*}
	\lVert g_t(\bar\theta_t) - g_t(\hat\theta_t)\rVert_{\mbold{\widetilde{V}_t^{-1}}} \leq \sqrt{\lambda}\cm S + \left\lVert\sum_{s=1}^{t-1}\gamma^{t-1-s}\epsilon_{s+1}x_s  \right\rVert_{\mbold{\widetilde{V}_t^{-1}}}\;.
\end{align*}	
Simplifying the factors $\gamma^{t-1}$ in the most right term and applying Proposition 1 of \citet{russac2019weighted} proves that with probability at least $1-\delta$, for all $t\geq 1$:
\begin{align*}
	\lVert g_t(\bar\theta_t) - g_t(\hat\theta_t)\rVert_{\mbold{\widetilde{V}_t^{-1}}} \leq \sqrt{\lambda}\cm S + \sigma\sqrt{2\log(1/\delta)+d\log\left(1+\frac{L^2(1-\gamma^{2t})}{\lambda d(1-\gamma^2)}\right)} = \beta_t(\delta)\;,
\end{align*}
hence proving the desired result.
\end{proof} 

\subsection{Bounding the prediction error}
 \begin{lemm}
 			Let $\delta\in(0,1]$ and $D\in\mathbb{N}^*$. With probability at least $1-\delta$: for all $t\geq 1$, for all $x\in\mcal{X}_t$, under Assumption
 			\ref{ass:ortho} the following holds.
 			\begin{align*}
 				\Delta_t(x) \leq \frac{2\km}{\cm}\beta_t(\delta) \lVert x\rVert_{\mathbf{V_t^{-1}}} \!+\! \frac{4\km^2 L^3 S}{\cm\lambda} \frac{\gamma^D}{(1-\gamma)} \!+\!  \frac{2\km^2L}{\cm} \sum_{s=t-D}^{t-1} \ltwo{\ts{s}-\ts{s+1}}.
 			\end{align*}
Without Assumption \ref{ass:ortho}, under general arm-set geometry, the following holds.
\begin{align*}
 \Delta_t(x) &\leq \frac{2\km}{\cm}\beta_t(\delta) \lVert x\rVert_{\mathbf{V_t^{-1}}} 
 \\
 & + \frac{2 \km L} {\cm} \sqrt{1 + \frac{L^2}{\lambda(1-\gamma)}} 
	    \left( \frac{2 \km S L^2}{\lambda} \frac{\gamma^D}{1-\gamma} + \km \sqrt{\frac{d}{\lambda(1-\gamma)}}
	    \sum_{s=t-D}^{t-1} \ltwo{\ts{s}-\ts{s+1}} \right)\; .
 			\end{align*}
 			\label{lemma:prediction_bound}
\end{lemm}
\begin{proof}
In the following, we assume that the event $E_\delta = \{\bar\theta_t \in \mcal{E}_t^\delta(\hat\theta_t)  \text{ for all } t\geq 1\}$ holds, which happens with probability at least $1-\delta$ (Lemma~\ref{lemma:confidence_set}). From the definition of the prediction error:
\begin{align}
	\Delta_t(x) &= \left\vert \mu(\langle x, \tilde\theta_t\rangle)-\mu(\langle x, \ts{t}\rangle)\right\vert&\notag \\
			&\leq \left(\sup_{x\in \mcal{X},\theta\in\Theta} \dot{\mu}\left(\langle x,\theta\rangle\right)\right) \left\vert \langle x,  \tilde\theta_t -\ts{t}\rangle\right\vert & (x \in\mcal{X}_t, \ts{t}\in\Theta, \tilde\theta_t\in\Theta) \notag\\
			&\leq \km \left\vert \langle x,  \tilde\theta_t -\ts{t}\rangle\right\vert\;. &(\text{by definition of } \km)
			\label{eq:delta_pref_lip}
\end{align}
Further, thanks to the mean value theorem:
\begin{align}
	g_t(\tilde\theta_t) - g_t(\ts{t}) &= \sum_{s=1}^{t-1} \gamma^{t-1-s}\left[\mu(\langle \tilde\theta_t,x_s\rangle)-\mu(\langle \ts{t},x_s\rangle)\right] + \lambda\cm( \tilde\theta_t -\ts{t})\notag\\
	&= \sum_{s=1}^{t-1} \gamma^{t-1-s}\left[\int_{v=0}^1 \dot{\mu}\left(\langle x_s, (1-v)\ts{t} + v\tilde\theta_t\rangle\right)dv\right] x_sx_s^\transp (\tilde\theta_t-\ts{t})+\lambda\cm( \tilde\theta_t -\ts{t})\notag\\
	&= \mathbf{G_t}\cdot(\tilde\theta_t-\ts{t})\;, \label{eq:diff_gt}
\end{align}
where:
\begin{align*}
	\mathbf{G_t} \defeq  &\sum_{s=1}^{t-1} \gamma^{t-1-s}\left[\int_{v=0}^1 \dot{\mu}\left(\langle x_s, (1-v)\ts{t} + v\tilde\theta_t\rangle\right) \, dv\right] 
	x_sx_s^\transp + \lambda\cm\mathbf{I_d}  \succeq \cm \mathbf{V_t} \;. 
\end{align*}
Note that because $x_s\in\mcal{X}$ for all $s\in[t-1]$ and $\tilde\theta_t,\ts{t}\in\Theta$ we have $\mathbf{G_t} \geq \cm\mathbf{V_t}$.
Assembling together Equations~\eqref{eq:delta_pref_lip} and \eqref{eq:diff_gt} we get:
\begin{align}
	\Delta_t(x) &\leq k_\mu \left\vert \left\langle  x,  \mathbf{G_t^{-1}}(g_t(\tilde\theta_t) - g_t(\ts{t})) \right\rangle \right\vert \notag\\
	&\leq k_\mu \left\vert \left\langle  x,  \mathbf{G_t^{-1}}(g_t(\tilde\theta_t) -g_t(\theta_t^p)+g_t(\theta_t^p)-g_t(\hat\theta_t)+g_t(\hat\theta_t)-g_t(\bar\theta_t)+g_t(\bar\theta_t)- g_t(\ts{t})) \right\rangle \right\vert \notag\\
	&\leq \underbrace{\km\left\vert \left\langle  x,  \mathbf{G_t^{-1}}(g_t(\tilde\theta_t) -g_t(\theta_t^p)+g_t(\hat\theta_t) -g_t(\bar\theta_t))\right\rangle \right\vert}_{\defeq\Delta_t^{\text{learn}}(x) }\notag \\&\qquad +\underbrace{ \km\left\vert \left\langle  x,  \mathbf{G_t^{-1}}(g_t(\theta_t^p) -g_t(\hat\theta_t)+g_t(\bar\theta_t) -g_t(\ts{t}))\right\rangle \right\vert}_{\defeq\Delta_t^{\text{track}}(x)}\notag\\
	&\leq \Delta_t^{\text{learn}}(x)+\Delta_t^{\text{track}}(x)\;.
	\label{eq:delta_decomposition}
\end{align}

This decomposition brings out the contribution of two different phenomenons (\emph{learning} and \emph{tracking}) which will be handled separately. Starting with the learning:

\begin{align}
	\Delta_t^{\text{learn}}(x) &= \km\left\vert \left\langle   x,  \mathbf{G_t^{-1}}(g_t(\tilde\theta_t) -g_t(\theta_t^p)+g_t(\hat\theta_t) -g_t(\bar\theta_t))\right\rangle \right\vert &\notag\\
	&= \km\left\vert \left\langle \mathbf{\widetilde{V}_t^{1/2}} \mathbf{G_t^{-1}}x,  \mathbf{\widetilde{V}_t^{-1/2}}(g_t(\tilde\theta_t) -g_t(\theta_t^p)+g_t(\hat\theta_t) -g_t(\bar\theta_t))\right\rangle \right\vert &\notag\\
	&\leq \km \lVert x\rVert_{\mathbf{G_t^{-1}\widetilde{V}_t}\mathbf{G_t^{-1}}}\left( \lVert g_t(\tilde\theta_t) -g_t(\theta_t^p)\rVert_{ \mathbf{\widetilde{V}_t^{-1}} }+\lVert g_t(\hat\theta_t) -g_t(\bar\theta_t)\rVert_{ \mathbf{\widetilde{V}_t^{-1}} }\right) &\text{(Cauchy-Schwarz)} 
	\notag\\
	&\leq \km \lVert x\rVert_{\mathbf{G_t^{-1}V_t}\mathbf{G_t^{-1}}}\left( \lVert g_t(\tilde\theta_t) -g_t(\theta_t^p)\rVert_{ \mathbf{\widetilde{V}_t^{-1}} }+\lVert g_t(\hat\theta_t) -g_t(\bar\theta_t)\rVert_{ \mathbf{\widetilde{V}_t^{-1}} }\right) &(\mathbf{\widetilde{V}_t}\leq \mathbf{V}_t) \notag\\
	&\leq \frac{\km}{\sqrt{\cm}} \lVert x\rVert_{\mathbf{G_t^{-1}}}\left( \lVert g_t(\tilde\theta_t) -g_t(\theta_t^p)\rVert_{ \mathbf{\widetilde{V}_t^{-1}} }+\lVert g_t(\hat\theta_t) -g_t(\bar\theta_t)\rVert_{ \mathbf{\widetilde{V}_t^{-1}} }\right) &(\mathbf{V_t}\leq \cm^{-1} \mathbf{G_t}) \notag\\
	&\leq \frac{\km}{\cm} \lVert x\rVert_{\mathbf{V_t^{-1}}}\left( \lVert g_t(\tilde\theta_t) -g_t(\theta_t^p)\rVert_{ \mathbf{\widetilde{V}_t^{-1}} }+\lVert g_t(\hat\theta_t) -g_t(\bar\theta_t)\rVert_{ \mathbf{\widetilde{V}_t^{-1}} }\right) &( \mathbf{G_t^{-1}} \leq \cm^{-1} \mathbf{V_t^{-1}}) \notag\\
	&\leq  \frac{\km}{\cm} \lVert x\rVert_{\mathbf{V_t^{-1}}}\left( \beta_t(\delta) +\lVert g_t(\hat\theta_t) -g_t(\bar\theta_t)\rVert_{ \mathbf{\widetilde{V}_t^{-1}} }\right) &(\tilde\theta_t\in\mcal{E}_t^\delta(\theta_t^p))\notag \\
	&\leq   \frac{\km}{\cm} \lVert x\rVert_{\mathbf{V_t^{-1}}}\left( \beta_t(\delta) + \beta_t(\delta)\right)\;. &(E_\delta \text{ holds})\notag
\end{align}
We used $\mathbf{\widetilde{V}_t}\leq \mathbf{V}_t$ which is a consequence of $\gamma\in(0,1)$. 
As a result:
\begin{equation}
		\Delta_t^{\text{learn}}(x) \leq \frac{2\km}{\cm}\beta_t(\delta) \lVert x\rVert_{\mathbf{V_t^{-1}}}\;.
	\label{eq:bound_delta_learn}
\end{equation}

The tracking term is bounded differently when the action set satisfies Assumption \ref{ass:ortho} or for 
general arm-set geometry. The bound on the tracking term is reported in Lemma~\ref{lemma:trackortho} and its proof is reported in 
Section~\ref{sec:prooflemma:drift}
%
%

\begin{restatable}{lemm}{lemmatrackortho}
\ifappendix
	Let $D\in\mbb{N}^*$. When Assumption~\ref{ass:ortho} holds, we have the following:
    \begin{equation}
	\Delta_t^{\text{track}}(x) \leq \frac{4\km^2 L^3 S}{\cm\lambda} \frac{\gamma^D}{(1-\gamma)} +  \frac{2\km^2 L}{\cm}\sum_{s=t-D}^{t-1} \ltwo{\ts{s}-\ts{s+1}} \;.
    \end{equation}
	For general arm-set geometry, we have the following
	\begin{equation*}
	    \Delta_t^{\text{track}}(x) \leq  \frac{2 \km L} {\cm} \sqrt{1 + \frac{L^2}{\lambda(1-\gamma)}} 
	    \left( \frac{2 \km S L^2}{\lambda} \frac{\gamma^D}{1-\gamma} + \km \sqrt{\frac{d}{\lambda(1-\gamma)}}
	    \sum_{s=t-D}^{t-1} \ltwo{\ts{s}-\ts{s+1}}
	    \right)
	\end{equation*}
	\label{lemma:trackortho}
\else
\fi
\end{restatable}

Assembling Equations~\eqref{eq:delta_decomposition}, \eqref{eq:bound_delta_learn} and the two different 
inequalities from Lemma~\ref{lemma:trackortho} gives the two statements of the proof. 
\end{proof}

\subsection{Proof of Lemma~\ref{lemma:trackortho}}
\label{sec:prooflemma:drift}
\lemmatrackortho*
\begin{proof}
Throughout the proof, we will use the following lemma, proven in Section~\ref{sec:proof_lemma:bound_second_drift}.
\lemmaboundseconddrift*
\noindent \underline{With Assumption 4.} In this first part of the proof, we assume that Assumption~\ref{ass:ortho} holds. 
We have the following:
\begin{align}
	\Delta_t^{\text{track}}(x) &= \km\left\vert \left\langle  x,  \mathbf{G_t}^{-1}(g_t(\theta_t^p) -g_t(\hat\theta_t)+g_t(\bar\theta_t) -g_t(\ts{t}))\right\rangle \right\vert  & \notag \\
	&\leq \km\ltwo{x}\left\lVert g_t(\theta_t^p) -g_t(\hat\theta_t)+g_t(\bar\theta_t) -g_t(\ts{t})\right\rVert_{\mathbf{G}_t^{-2}}&\text{(Cauchy-Schwarz)}
	\notag \\
	&\leq \frac{\km L}{\cm} \left\lVert g_t(\theta_t^p) -g_t(\hat\theta_t)+g_t(\bar\theta_t) -g_t(\ts{t})\right\rVert_{\mathbf{V}_t^{-2}}&(\ltwo{x}\leq L, \mathbf{G_t}^2\succeq \cm^2\mathbf{V_t}^2)\notag \\
	&\leq  \frac{\km L}{\cm} \left( \left\lVert g_t(\theta_t^p) -g_t(\hat\theta_t)\right\rVert_{\mathbf{V}_t^{-2}}+\left\lVert g_t(\bar\theta_t) -g_t(\ts{t})\right\rVert_{\mathbf{V}_t^{-2}}\right)&\text{(Triangle inequality)}\notag \\
	&\leq \frac{2\km L}{\cm}\left\lVert g_t(\bar\theta_t) -g_t(\ts{t})\right\rVert_{\mathbf{V}_t^{-2}} &\text{(Lemma~\ref{lemma:bound_second_drift}})\notag
\end{align}
where the third inequality can be obtained only because when Assumption~\ref{ass:ortho} holds $\mbold{G}_t$ and $\mbold{V}_t$ commute.
The final is obtained using Lemma~\ref{lemma:bound_first_drift} reported here and established in 
Section~\ref{lemma:bound_first_drift}
\lemmaboundfirstdrift*
\noindent\underline{Without Assumption 4.} We now explain how to extend the analysis with general arm-set geometry.
\begin{align}
	\Delta_t^{\text{track}}(x) &= \km\left\vert \left\langle  x,  \mathbf{G_t}^{-1}(g_t(\theta_t^p) -g_t(\hat\theta_t)+g_t(\bar\theta_t) -g_t(\ts{t}))\right\rangle \right\vert  & \notag \\
	& = \km\left\vert \left\langle  x,  \mathbf{G_t}^{-1}
	\mathbf{V_t} \mathbf{V_t}^{-1}
	(g_t(\theta_t^p) -g_t(\hat\theta_t)+g_t(\bar\theta_t) -g_t(\ts{t}))\right\rangle \right\vert  & \notag \\
	&\leq \km \lVert x \rVert_{\mathbf{G_t}^{-1} \mathbf{V_t}^{2} \mathbf{G_t}^{-1}}
	\left\lVert g_t(\theta_t^p) -g_t(\hat\theta_t)+g_t(\bar\theta_t) -g_t(\ts{t})\right\rVert_{\mathbf{V}_t^{-2}} \quad \text{(Cauchy-Schwarz)}
	\notag \\
	&\leq \km  \sqrt{\lambda_{\textnormal{max}}(\mathbf{V_t})} 
	\lVert x \rVert_{\mathbf{G_t}^{-1} \mathbf{V_t} \mathbf{G_t}^{-1}}
	\left\lVert g_t(\theta_t^p) -g_t(\hat\theta_t)+g_t(\bar\theta_t) -g_t(\ts{t})\right\rVert_{\mathbf{V}_t^{-2}}\notag \\
	&\leq \frac{\km}{\sqrt{\cm}} \sqrt{\lambda_{\textnormal{max}}(\mathbf{V_t})} 
	\lVert x \rVert_{\mathbf{G_t}^{-1}} \left\lVert g_t(\theta_t^p) -g_t(\hat\theta_t)+g_t(\bar\theta_t) -g_t(\ts{t})\right\rVert_{\mathbf{V}_t^{-2}} \quad
	(\mathbf{G_t}\succeq \cm\mathbf{V_t}) 
	\notag\\
	&\leq \frac{\km L }{\sqrt{\lambda} \cm} \sqrt{\lambda_{\textnormal{max}}(\mathbf{V_t})} 
	\left\lVert g_t(\theta_t^p) -g_t(\hat\theta_t)+g_t(\bar\theta_t) -g_t(\ts{t})\right\rVert_{\mathbf{V}_t^{-2}} \quad 
	(\ltwo{x}\leq L, \mathbf{G_t}\succeq \lambda \cm\mathbf{I_d})
	\notag\\ 
	&\leq  \frac{\km L}{\sqrt{\lambda} \cm} \sqrt{\lambda_{\textnormal{max}}(\mathbf{V_t})}  
	\left( \left\lVert g_t(\theta_t^p) -g_t(\hat\theta_t)\right\rVert_{\mathbf{V}_t^{-2}}+\left\lVert g_t(\bar\theta_t) -g_t(\ts{t})\right\rVert_{\mathbf{V}_t^{-2}}\right)\quad \text{(Triangle inequality)}\notag \\
	&\leq  \frac{ 2 \km L}{\sqrt{\lambda} \cm} \sqrt{\lambda_{\textnormal{max}}(\mathbf{V_t})}  
	\left\lVert g_t(\bar\theta_t) -g_t(\ts{t})\right\rVert_{\mathbf{V}_t^{-2}}
	\quad \text{(Lemma~\ref{lemma:bound_second_drift}})\notag
\end{align}
We then use:
\begin{equation}
    \lambda_{\textnormal{max}}(\mathbf{V_t}) \leq \frac{L^2}{1-\gamma} + \lambda\; .
\end{equation}
That can be obtained by computing the operator norm of the matrix $\mathbf{V_t}$.
Combining this with Lemma~\ref{lemma:bound_first_drift_bis} reported here and proved in 
Section~\ref{sec:prooflemma:bound_first_drift_bis} achieves the proof.

\begin{restatable}{lemm}{lemmaboundfirstdriftbis}[Tracking with general action sets]
	Let $D\in\mbb{N}^*$. The following holds:
	\begin{align*}
		\lVert g_t(\bar\theta_t) -g_t(\ts{t})\rVert_{\mathbf{V_t^{-2}}} \leq \frac{2\km L^2 S}{\lambda} \frac{\gamma^D}{1-\gamma} + \frac{\km}{\sqrt{\lambda}} \frac{\sqrt{d}}{\sqrt{1-\gamma}} \sum_{s=t-D}^{t-1} \ltwo{\ts{s}-\ts{s+1}}.
	\end{align*}
	\label{lemma:bound_first_drift_bis}
\end{restatable}
\end{proof}

\subsection{Proof of Lemma~\ref{lemma:bound_second_drift}}
\label{sec:proof_lemma:bound_second_drift}
\lemmaboundseconddrift*

\begin{proof}
	We prove this result by \underline{contradiction}. Assume that:
	\begin{align}
		\lVert g_t(\theta_t^p) -g_t(\hat\theta_t)\rVert_{\mathbf{V_t^{-2}}} > \lVert g_t(\bar\theta_t) -g_t(\ts{t})\rVert_{\mathbf{V_t^{-2}}},
		\label{eq:contradiction}
	\end{align}
	For all $s\geq 1$ define:
	\begin{align}
		\tilde{r}_{s+1} \defeq \mu(\langle x_s, \ts{t}\rangle) + \epsilon_{s+1} \;,
		\label{eq:rtilde_def}
	\end{align}
	where $\{\epsilon_s\}_{s}$ is defined in Equation~\eqref{eq:eta_def}. Further, let:
	\begin{align*}
		\theta_c \defeq  \argmin_{\theta\in\mathbb{R}^d} \sum_{s=1}^{t-1} \gamma^{t-1-s} \left[b(\langle \theta, x_s\rangle) - \tilde{r}_{s+1}\langle \theta, x_s\rangle)\right] + \frac{\lambda\cm}{2}\ltwo{\theta}^2\;,
	\end{align*}
	which is well-defined as the minimizer of a strictly convex, coercive function. Upon differentiating we get:
	\begin{align}
		g_t(\theta_c) &= \sum_{s=1}^{t-1} \gamma^{t-1-s}\tilde{r}_{s+1}x_s & \notag \\
		&= \sum_{s=1}^{t-1} \gamma^{t-1-s}\epsilon_{s+1}x_s  + \sum_{s=1}^{t-1} \gamma^{t-1-s}\mu(\langle x_s,\ts{t} \rangle)x_s &\text{(Equation~\eqref{eq:rtilde_def})} \notag\\
		&= g_t(\hat\theta_t) - g_t(\bar\theta_t) + \lambda\cm\ts{t} + \sum_{s=1}^{t-1} \gamma^{t-1-s}\mu(\langle x_s,\ts{t} \rangle)x_s &\text{(Equation~\eqref{eq:noideafornamee})} \notag\\
		&= g_t(\hat\theta_t) - g_t(\bar\theta_t)+ g_t(\ts{t})\;. &\label{eq:diffgt_thetac}
	\end{align}
	Therefore:
	\begin{align*}
		\lVert  g_t(\theta_c) - g_t(\hat\theta_t) \rVert_{\mathbf{V_t^{-2}}} &= \lVert g_t(\bar\theta_t)- g_t(\ts{t})\rVert_{\mathbf{V_t^{-2}}}&\\ 
		&< \lVert g_t(\theta_t^p) -g_t(\hat\theta_t)\rVert_{\mathbf{V_t^{-2}}}\;.& \text{(Equation~\ref{eq:contradiction})}
	\end{align*}
	Further from Equation~\eqref{eq:diffgt_thetac} we get:
	\begin{align*}
		\lVert  g_t(\theta_c) - g_t(\ts{t}) \rVert_{\mathbf{\tilde{V}_t^{-1}}} &= \lVert  g_t(\bar\theta_t) - g_t(\hat\theta_t) \rVert_{\mbold{\tilde{V}_t^{-1}}} &\\
		  &\leq \beta_t(\delta) &(\bar\theta_t\in\mcal{E}_t^{\delta}(\hat\theta_t) )\\
		&\pmb{\Leftrightarrow} \ts{t}\in\mcal{E}_t^\delta(\theta_c)\;.
	\end{align*}
	
	To sum-up, we have $\lVert  g_t(\theta_c) - g_t(\hat\theta_t) \rVert_{\mathbf{V_t^{-2}}} < \lVert g_t(\theta_t^p) -g_t(\hat\theta_t)\rVert_{\mathbf{V_t^{-2}}}$ and $\mcal{E}_t^\delta(\theta_c)\cap\Theta\neq\emptyset$ since $\ts{t}\in\Theta \cap \mcal{E}_t^\delta(\theta_c)$. This contradicts the definition of $\theta_t^p$ (in \eqref{opt:program_0}) and therefore Equation~\eqref{eq:contradiction} must be wrong, which proves the announced result.
\end{proof}

\subsection{Proof of Lemma~\ref{lemma:bound_first_drift}}
\label{sec:prooflemma:bound_first_drift}
\lemmaboundfirstdrift*
\begin{proof}
	Thanks to Equation~\eqref{eq:theta_bar_gt} we have:
	\begin{align*}
		&g_t(\bar\theta_t) = \sum_{s=1}^{t-1} \gamma^{t-1-s}\mu(\langle x_s,\ts{s}\rangle)x_s + \lambda\cm \ts{t}&\\
		&\pmb{\Leftrightarrow} g_t(\bar\theta_t)-g_t(\ts{t}) =  \sum_{s=1}^{t-1}\gamma^{t-1-s} \left[\mu(\langle x_s,\ts{s}\rangle)-\mu(\langle x_s,\ts{t}\rangle)\right]x_s &\\
		 &\pmb{\Leftrightarrow} g_t(\bar\theta_t)-g_t(\ts{t}) =  \sum_{s=1}^{t-1}\gamma^{t-1-s} \left[\int_{v=0}^1 \dot{\mu}\left(\left\langle x_s, v\ts{t}+(1-v)\ts{s}\right\rangle\right)dv\right]x_sx_s^\transp(\ts{s}-\ts{t}) &\text{(mean-value theorem)}\\
		 &\pmb{\Leftrightarrow} g_t(\bar\theta_t)-g_t(\ts{t}) =  \sum_{s=1}^{t-1}\gamma^{t-1-s} \alpha_sx_sx_s^\transp(\ts{s}-\ts{t})\;,
	\end{align*}
	where we defined:
	\begin{align*}
		\alpha_s \defeq \int_{v=0}^1 \dot{\mu}\left(\left\langle x_s, v\ts{t}+(1-v)\ts{s}\right\rangle\right)dv \in [\cm,\km]\;.
	\end{align*}
	Therefore:
	\begin{align}
	\label{eq:usefull_g_t}
		\lVert g_t(\bar\theta_t) -g_t(\ts{t})\rVert_{\mathbf{V_t^{-2}}} &= \left\lVert \sum_{s=1}^{t-1}\gamma^{t-1-s} \alpha_sx_sx_s^\transp(\ts{s}-\ts{t})\right\rVert_{\mathbf{V_t^{-2}}}
		\;.
	\end{align}
	The rest of the proof follows the strategy of \citet{russac2019weighted} to yield the announced result. Let $D\in\mbb{N}^*$ and notice that:
	\begin{align*}
		\left\lVert \sum_{s=1}^{t-1}\gamma^{t-1-s} \alpha_s x_sx_s^\transp(\ts{s}-\ts{t})\right\rVert_{\mathbf{V_t^{-2}}} &\leq \underbrace{\left\lVert \sum_{s=1}^{t-D-1}\gamma^{t-1-s} \alpha_s  x_sx_s^\transp(\ts{s}-\ts{t})\right\rVert_ \mathbf{V_t^{-2}}}_{:=d_1}   \\& +  \underbrace{\left\lVert\sum_{s=t-D}^{t-1}\gamma^{t-1-s}\alpha_s  x_sx_s^\transp(\ts{s}-\ts{t})\right\rVert_\mathbf{V_t^{-2}} }_{:=d_2}\;.
	\end{align*}
	Both terms are bounded separately; starting with $d_1$:
	\begin{align*}
	d_1  &\leq \lambda^{-1} \left\lVert \sum_{s=1}^{t-D-1}\gamma^{t-1-s} \alpha_s x_sx_s^\transp(\ts{s}-\ts{t})\right\rVert &(\mathbf{V_t}\geq \lambda\mathbf{I_d})\\
	   &\leq \lambda^{-1} \sum_{s=1}^{t-D-1} \gamma^{t-1-s} \vert \alpha_s \vert \left\lVert x_sx_s^\transp(\ts{s}-\ts{t})\right\rVert &\text{(Triangle inequality)} \\
	   &\leq 2\km \lambda^{-1}SL^2 \sum_{s=1}^{t-D-1} \gamma^{t-1-s}&(\ltwo{x_s}\leq L, \, \ts{s},\ts{t}\in\Theta, \vert \alpha_s\vert \leq \km)\\
	   &\leq 2\km\lambda^{-1}S L^2 \gamma^D(1-\gamma)^{-1} \;.
	\end{align*}
	For $d_2$ a careful analysis is required.
	\begin{align*}
		d_2 &=  \left\lVert \mathbf{V_t^{-1}}\sum_{s=t-D}^{t-1}\gamma^{t-1-s} \alpha_s x_sx_s^\transp(\ts{s}-\ts{t})\right\rVert &\\
		&= \left\lVert \sum_{s=t-D}^{t-1} \mathbf{V_t^{-1}}\gamma^{t-1-s} \alpha_s x_sx_s^\transp(\ts{s}-\ts{t})\right\rVert&\\
		&=  \left\lVert \sum_{s=t-D}^{t-1} \mathbf{V_t^{-1}}\gamma^{t-1-s} \alpha_s x_sx_s^\transp\sum_{p=s}^{t-1}\left( \ts{p}-\ts{p+1}\right)\right\rVert &\text{(Telescopic sum)}\\ 
		&\leq \left\lVert \sum_{p=t-D}^{t-1} \mathbf{V_t^{-1}}
		\sum_{s=t-D}^{p} \gamma^{t-1-s}  \alpha_s x_sx_s^\transp
		\left( \ts{p}-\ts{p+1}\right)\right\rVert &\text{(Re-arranging)}\\
		&\leq \sum_{p=t-D}^{t-1}   \left\lVert \mathbf{V_t^{-1}}\sum_{s=t-D}^{p} \gamma^{t-1-s} \alpha_s x_sx_s^\transp\left( \ts{p}-\ts{p+1}\right)\right\rVert &\text{(Triangle inequality)}
	\end{align*}
	
	At this point, Assumption~\ref{ass:ortho} can be used to upper-bound the operator norm of the matrix 
	$\mathbf{V_t^{-1}}\sum_{s=t-D}^{p} \gamma^{t-1-s} \alpha_s x_sx_s^\transp$.
	Under Assumption~\ref{ass:ortho}, the following holds:
	\begin{equation}
	\label{eq:matrix_link}
	\mathbf{V_t^{-1}}\sum_{s=t-D}^{p} \gamma^{t-1-s} \alpha_s x_sx_s^\transp
	= \mathbf{V_t^{-1/2}}\sum_{s=t-D}^{p} \gamma^{t-1-s} \alpha_s x_sx_s^\transp \mathbf{V_t^{-1/2}}
	:= \mathbf{M_t}
	\end{equation}
	The advantage, now is that the matrix on the right-hand side of Equation~\eqref{eq:matrix_link} is
	symmetric and we can use the relation $\lVert M x \rVert \leq \lVert M \rVert \, \lVert x\rVert_2$ that holds for all symmetric matrix $M$ and where $\lVert M \rVert$ denotes the operator norm of $M$.
	The final step consists in upper-bounding the operator norm of  $\mathbf{M_t}$.
	We have, 
	\begin{align*}
	    \forall x, \lVert x \rVert_2 \leq 1, \quad x^\transp \mathbf{M_t} x
	    &= x^\transp \mathbf{V_t^{-1/2}}\sum_{s=t-D}^{p} \gamma^{t-1-s} \alpha_s x_sx_s^\transp \mathbf{V_t^{-1/2}} x \\ 
	    &= \sum_{s=t-D}^p  \alpha_s x^\transp \mathbf{V_t^{-1/2}} x_s x_s^\transp \mathbf{V_t^{-1/2}} x \\
	    &= \sum_{s=t-D}^{p}\gamma^{t-1-s}\alpha_s  \left(x_s^\transp \mathbf{V_t^{-1/2}} x\right)^2 \\
	    &\leq \km \sum_{s=t-D}^{p}\gamma^{t-1-s}  \left(x_s^\transp \mathbf{V_t^{-1/2}} x\right)^2 \\ 
	    & \leq \km x^\transp \mathbf{V_t^{-1/2}}\sum_{s=t-D}^{p} \gamma^{t-1-s} x_sx_s^\transp \mathbf{V_t^{-1/2}} x\;.
	\end{align*}
	Furthermore,
	\begin{align*}
	    \forall x, \lVert x \rVert_2 \leq 1, \quad
	    x^\transp \mathbf{V_t^{-1/2}}\sum_{s=t-D}^{p} \gamma^{t-1-s} x_s x_s^\transp \mathbf{V_t^{-1/2}} x
	    &\leq x^\transp \mathbf{V_t^{-1/2}} \left( \sum_{s=1}^{t-1} \gamma^{t-1-s} x_s x_s^\transp \right)
	    \mathbf{V_t^{-1/2}} x \\
	    &\leq  x^\transp x \leq 1 \;.
	\end{align*}
	Combining the two inequalities ensures that 
	\begin{equation}
	    \label{eq:useful_matrix_2}
	    \lVert \mathbf{M_t} \rVert \leq \km
	\end{equation}
	
	Finally,
	
	\begin{align*}
	    d_2 \leq \km \sum_{p=t-D}^{t-1}\lVert \ts{p}-\ts{p+1}\rVert \;.
	\end{align*}
\end{proof}

\subsection{Proof of Lemma~\ref{lemma:bound_first_drift_bis}}
\label{sec:prooflemma:bound_first_drift_bis}
\lemmaboundfirstdriftbis*
\begin{proof}
    Following the proof of Lemma~\ref{lemma:bound_first_drift}, one has:
    
    \begin{align*}
    \lVert g_t(\bar\theta_t) -g_t(\ts{t})\rVert_{\mathbf{V_t^{-2}}} &= \left\lVert \mathbf{V_t^{-1}} \sum_{s=1}^{t-1}\gamma^{t-1-s} \alpha_sx_sx_s^\transp(\ts{s}-\ts{t})\right\rVert_{2} 
    \end{align*}
    We follow the line of proof from \citep[Appendix C]{touati2020efficient} where the only difference is the $\alpha_s$ term. We use 
    \begin{equation}
        \label{eq:proof_lemma_7}
        \left\lVert \mathbf{V_t^{-1}} \sum_{s=1}^{t-1}\gamma^{t-1-s} \alpha_sx_sx_s^\transp(\ts{s}-\ts{t})\right\rVert_{2}  = \max_{x: \lVert x \rVert_2=1} 
        \left
        |x^\transp \mathbf{V_t^{-1}} \sum_{s=1}^{t-1}\gamma^{t-1-s} \alpha_sx_sx_s^\transp(\ts{s}-\ts{t})
        \right|
    \end{equation}

Let $x \in \mathbb{R}^d$ such that $\lVert x \rVert_2= 1$, we have
\begin{align*}
    \left|x^\transp \mathbf{V_t^{-1}} \sum_{s=1}^{t-1}\gamma^{t-1-s} \alpha_sx_sx_s^\transp(\ts{s}-\ts{t})
    \right| &\leq
    \left|x^\transp \mathbf{V_t^{-1}} \sum_{s=1}^{t-D-1}\gamma^{t-1-s} \alpha_sx_sx_s^\transp(\ts{s}-\ts{t})
    \right| \\
    & +  
    \left|x^\transp \mathbf{V_t^{-1}} \sum_{s=t-D}^{t-1}\gamma^{t-1-s} \alpha_sx_sx_s^\transp(\ts{s}-\ts{t})
    \right|
\end{align*}

For the first term, using Cauchy-Schwarz, we obtain the term $d_1$ from the proof of Lemma~\ref{lemma:bound_first_drift}.
Hence,
\begin{align*}
    \left|x^\transp \mathbf{V_t^{-1}} \sum_{s=1}^{t-D-1}\gamma^{t-1-s} \alpha_s x_s x_s^\transp(\ts{s}-\ts{t})
    \right| \leq \lVert x \rVert_2 \frac{2 \km S L^2}{\lambda} \frac{\gamma^D}{1- \gamma}
\end{align*}

Let $b = \left|x^\transp \mathbf{V_t^{-1}} \sum_{s=t-D}^{t-1}\gamma^{t-1-s} \alpha_sx_sx_s^\transp(\ts{s}-\ts{t})
    \right|$. One has,
    
\begin{align*}
    b &=  \sum_{s= t-D}^{t-1} \gamma^{t-1-s} |x^\transp \mathbf{V_t}^{-1} x_s | |\alpha_s| |x_s^\transp(\theta^s_\star-
    \theta_\star^t)| \quad  \textnormal{(Triangle inequality)} \\
    & \leq \km \sum_{s= t-D}^{t-1} |x^\transp \mathbf{V_t}^{-1} x_s | |x_s^\transp \sum_{p=s}^{t-1}(\theta^p_\star-
    \theta_\star^{p+1})| \\
    & \leq \km L  \sum_{s= t-D}^{t-1} |x^\transp \mathbf{V_t}^{-1} x_s| \left\lVert \sum_{p=s}^{t-1} (\theta^p_\star-
    \theta_\star^{p+1}) \right\rVert_2 \quad \text{(Cauchy-Schwarz, } \lVert x_s\rVert_2 \leq L) \\
    & \leq \sum_{p = t-D}^{t-1} \sum_{s= t-D}^p \gamma^{t-1-s} |x^\transp \mathbf{V_t}^{-1} x_s|
    \lVert \theta^p_\star-
    \theta_\star^{p+1} \rVert_2 \\
    & \leq \sum_{p = t-D}^{t-1} \sum_{s= t-D}^p \gamma^{t-1-s} \sqrt{x^\transp \mathbf{V_t}^{-1} x}
    \sqrt{x_s^\transp \mathbf{V_t}^{-1} x_s}
    \lVert \theta^p_\star-
    \theta_\star^{p+1} \rVert_2 \quad \text{(Cauchy-Schwarz)} \\
    & \leq \sum_{p = t-D}^{t-1} \sqrt{\sum_{s=t-D}^{t-1} \gamma^{t-1-s} x^\transp \mathbf{V_t}^{-1} x}
    \sqrt{\sum_{s=t-D}^{t-1} \gamma^{t-1-s} x_s^\transp \mathbf{V_t}^{-1} x_s} \lVert \theta^p_\star-
    \theta_\star^{p+1} \rVert_2 \; \text{(Cauchy-Schwarz)}
\end{align*}

Now, 
\begin{align*}
    \sqrt{\sum_{s=t-D}^{t-1} \gamma^{t-1-s} x_s^\transp \mathbf{V_t}^{-1} x_s} &\leq
    \sqrt{\text{tr}\left(\sum_{s=t-D}^{t-1} \gamma^{t-1-s} x_s^\transp \mathbf{V_t}^{-1} x_s\right)} \\
    & \leq \sqrt{\text{tr}\left( \mathbf{V_t}^{-1} \sum_{s=t-D}^{t-1} \gamma^{t-1-s} x_s x_s^\transp \right)}
    & \leq \sqrt{\text{tr}(\mathbf{I_d})} = \sqrt{d}
\end{align*}
Further, 
\begin{align*}
     \sqrt{\sum_{s=t-D}^{t-1} \gamma^{t-1-s} x^\transp \mathbf{V_t}^{-1} x} &\leq
     \frac{1}{\sqrt{\lambda}} \lVert x \rVert_2 \frac{1}{\sqrt{1-\gamma}}
\end{align*}

Bringing things together, yields the announced result.
\end{proof}
\section{Regret bound}
\label{app:regret_bound}

\subsection{Regret decomposition}
\lemmaregretdecomposition*
\label{app:regret_decomposition}

\begin{proof}
We recall that $x_\star^t = \argmax_{x\in\mcal{X}_t} \mu(\langle \ts{t},x\rangle)$.
Note that:
\begin{equation*}
\begin{split}
	R_T &= \sum_{t=1}^{T}  \mu(\langle x_\star^t, \ts{t}\rangle)-\mu(\langle x_t, \ts{t}\rangle)\\
	        &= \sum_{t=1}^{T}  \mu(\langle x_\star^t, \ts{t}\rangle)-\mu(\langle x_\star^t, \tilde\theta_t\rangle)+\mu(\langle x_\star^t, \tilde\theta_t\rangle)- \mu(\langle x_t, \tilde\theta_t\rangle)+\mu(\langle x_t, \tilde\theta_t\rangle)-\mu(\langle x_t, \ts{t}\rangle)\\
    	        &=  \sum_{t=1}^{T}\left[\mu(\langle x_\star^t, \tilde\theta_t\rangle)- \mu(\langle x_t, \tilde\theta_t\rangle)\right] + \sum_{t=1}^{T}\left[\mu(\langle x_\star^t, \ts{t}\rangle) - \mu(\langle x_\star^t, \tilde\theta_t\rangle)\right] + \sum_{t=1}^{T}\left[\mu(\langle x_t, \tilde\theta_t\rangle)-\mu(\langle x_t, \ts{t}\rangle)\right]\\
	        &\leq \frac{2k_\mu}{c_\mu}\sum_{t=1}^{T} \beta_t(\delta)\left[\lVert x_t\rVert_{\mathbf{V_t^{-1}}}-\lVert x_\star^t\rVert_{\mathbf{V_t^{-1}}}\right] \\
	        & \quad 
	        + \sum_{t=1}^{T}\left[\mu(\langle x_\star^t, \ts{t}\rangle) - \mu(\langle x_\star^t, \tilde\theta_t\rangle)\right] + \sum_{t=1}^{T}\left[\mu(\langle x_t, \tilde\theta_t\rangle)-\mu(\langle x_t, \ts{t}\rangle)\right] \;.
\end{split}
\end{equation*}
In the last inequality, we used the fact that $x_t = \argmax_{x\in\mcal{X}}\left\{ \mu(\langle x,\tilde\theta_t\rangle)+ \frac{2k_\mu}{c_\mu}\beta_t(\delta)\lVert x\rVert_{\mathbf{V_t^{-1}}}\right\}$. Using the definition of $\Delta_t(x)$ we conclude that:
\begin{align*}
	R_T \leq \frac{2k_\mu}{c_\mu}\sum_{t=1}^{T} \beta_t(\delta)\left[\lVert x_t\rVert_{\mathbf{V_t^{-1}}}-\lVert x_\star^t\rVert_{\mathbf{V_t^{-1}}}\right] +\sum_{t=1}^{T}\left[\Delta_t(x_t)+\Delta_t(x_\star^t)\right]\;.
\end{align*}
\end{proof}

\subsection{Regret bound}
\label{app:complete_regret_bound}
We now claim Theorem~\ref{thm:regret_bound}, bounding the regret of \ouralgo. 
\thmregretbound*
\begin{proof}
In the following, we assume that the event $\{\bar\theta_t\in\mcal{E}_t^{\delta}(\hat\theta_t), \forall t\geq 1\}$ holds, which happens with probability at least $1-\delta$ (Lemma~\ref{lemma:confidence_set}). Thanks to Lemma~\ref{lemma:prediction_bound}, when Assumption~\ref{ass:ortho} the following holds:
\begin{align*}
	\Delta_t(x_t) + \frac{2\km}{\cm}\beta_t(\delta)\lVert x_t\rVert_{\mathbf{V_t^{-1}}} &\leq 
	\frac{4\km}{\cm}\beta_t(\delta)\lVert x_t\rVert_{\mathbf{V_t^{-1}}}+ \frac{4\km^2 L^3 S}{\cm\lambda(1-\gamma)} \gamma^D +  \frac{2\km^2 L}{\cm}\sum_{s=t-D}^{t-1} \ltwo{\ts{s}-\ts{s+1}}\\
	\Delta_t(x_\star^t) - \frac{2\km}{\cm}\beta_t(\delta)\lVert x_\star^t\rVert_{\mathbf{V_t^{-1}}} &\leq 
	 \frac{4\km^2 L^3 S}{\cm\lambda(1-\gamma)} \gamma^D +  \frac{2\km^2 L}{\cm}\sum_{s=t-D}^{t-1} \ltwo{\ts{s}-\ts{s+1}}
\end{align*}
Assembling this result with Lemma~\ref{lemma:regret_decomposition} yields:
\begin{align*}
	R_T \leq \underbrace{\sum_{t=1}^T \frac{4\km}{\cm}\beta_t(\delta)\lVert x_t\rVert_{\mathbf{V_t^{-1}}}}_{R_T^{\text{learn}}}+  \underbrace{\sum_{t=1}^T \left[\frac{8\km^2 L^3 S}{\cm\lambda(1-\gamma)} \gamma^D +  \frac{4\km^2 L}{\cm}\sum_{s=t-D}^{t-1} \ltwo{\ts{s}-\ts{s+1}}\right] }_{R_T^\text{track}} \;.
\end{align*}
We now bound each term separately. Starting with $R_T^{\text{learn}}$:
 \begin{align*}
 	R_T^{\text{learn}} &\leq  \frac{4\km}{\cm}\beta_T(\delta)\sum_{t=1}^T \lVert x_t\rVert_{\mathbf{V_t^{-1}}} &(t\to\beta_t(\delta) \text{ increasing})\\
	&\leq \frac{4\km}{\cm}\beta_T(\delta)\sqrt{T}\sqrt{\sum_{t=1}^T \lVert x_t\rVert^2_{\mathbf{V_t^{-1}}}}  &\text{(Cauchy-Schwarz)}\\
	&\leq  \frac{4\km}{\cm}\beta_T(\delta)\sqrt{2T\max(1,L^2/\lambda)}\sqrt{dT\log(1/\gamma) +\log\left(\frac{\det\mathbf{V_{T+1}}}{\lambda^d}\right)} &\text{(Lemma~\ref{lemma:ellipticalpotential})}\\
	  &\leq \frac{4\km}{\cm}\beta_T(\delta)\sqrt{2dT\max(1,L^2/\lambda)}\sqrt{T\log(1/\gamma) +\log\left(1+\frac{L^2(1-\gamma^{T})}{\lambda d(1-\gamma)}\right)} \;. &\text{(Lemma~\ref{lemma:determinant_trace_inequality})} 
 \end{align*}
The bounding of the tracking term is straight-forward:
 \begin{align*}
 	R_T^{\text{track}} &= \frac{8\km^2 L^3 S}{\cm\lambda(1-\gamma)} \gamma^DT + \frac{4\km^2 L}{\cm}\sum_{t=1}^T\sum_{s=t-D}^{t-1} \ltwo{\ts{s}-\ts{s+1}} \\
	&\leq \frac{8\km^2 L^3 S}{\cm\lambda(1-\gamma)} \gamma^DT + \frac{4\km^2 L}{\cm}DB_T \;.
 \end{align*}
 Assembling this two bounds ($R_T^{\text{learn}}$ and $R_T^{\text{track}}$)  yields the first announced result, with the following constants:
 \begin{align*}
 	C_1 &=  \sqrt{32\max(1,L^2/\lambda)}\;. \\
	C_2 &= \frac{8\km L^3 S}{\lambda}\;. \\
	C_3 &=  4\km L \;.
 \end{align*}
 The last part of the proof follows the asymptotic argument of \citet{russac2019weighted}. We assume that $B_T$ is sub-linear and let:
 \begin{align*}
 	D = \frac{\log T}{1-\gamma}\;, \qquad \gamma =1- \left( \frac{B_T}{dT}\right)^{2/3} \;.
 \end{align*}
 We therefore have the following asymptotic equivalences (omitting logarithmic dependencies):
 \begin{align*}
 	\beta_T(\delta)\sqrt{dT}\sqrt{T\log(1/\gamma)} &\sim dT\cdot \left( \frac{B_T}{dT}\right)^{1/3} &= d^{2/3}B_T^{1/3}T^{2/3}\\
	\gamma^D T/(1-\gamma)&\sim  \exp(-\log T)T\left(\frac{B_T}{dT}\right)^{-2/3}  &= d^{2/3}B_T^{-2/3}T^{2/3}\\
	DB_T &\sim B_T \left(\frac{B_T}{dT} \right)^{-2/3} &= d^{2/3}B_T^{1/3}T^{2/3}
 \end{align*}
 Merged with the regret-bound we just proved, this yields the announced result.

 \noindent 
 Without Assumption~\ref{ass:ortho} similar results can be obtained. The main difference consists in using the upper-bound 
 for the tracking term under general arm-set geometry which is slightly more complicated. Plugging the bound from 
 Lemma~\ref{lemma:prediction_bound} and upper bounding $\sqrt{1 + \frac{L^2}{\lambda(1-\gamma)}}$ by 
 $1+ \frac{L}{\sqrt{\lambda(1-\gamma)}}$ gives the announced regret decomposition.
Let:
$$
\gamma = 1- \frac{B_T^{2/5}}{d^{1/5} T^{2/5}}
$$

 \begin{align*}
 	\beta_T(\delta)\sqrt{dT}\sqrt{T\log(1/\gamma)} &\sim dT\cdot \frac{B_T^{1/5} d^{-1/10}}{T^{1/5}} &= d^{9/10}B_T^{1/5}T^{4/5}\\
	\gamma^D T/(1-\gamma)^{3/2} &\sim  \exp(-\log T)T \left(\frac{d^{1/5} T^{2/5} }{B_T^{2/5}}\right)^{3/2} 
	&= d^{3/10}B_T^{-3/5}T^{3/5}\\
	\frac{\sqrt{d}}{1-\gamma}DB_T &\sim d^{1/2} B_T \left(\frac{d^{1/5} T^{2/5} }{B_T^{2/5}}\right)^{2} &= d^{9/10}B_T^{1/5}T^{4/5}
 \end{align*}
\end{proof}

\section{On the projection step}
\label{app:equivalent_min_proof}
\subsection{Equivalent minimization program}
Recall the original minimization program for finding $\theta_t^p$:
\begin{align}
	\theta_t^p \in \argmin_{\theta\in\mbb{R}^d}\left\{ \left\lVert g_t(\theta)-g_t(\hat\theta_t) \right\rVert_{\mathbf{V^{-2}_t}} \text{ s.t } \Theta\cap\mcal{E}_t^\delta(\theta)\neq \emptyset \right\}\tag{\textbf{P1}} \;.
	\label{opt:program_1}
\end{align}
Note that this minimum exists ($0_d$ is feasible) and is indeed attained (the feasible set is compact and the objective smooth). The following reformulation is motivated by the fact that only $\tilde\theta_t\in \Theta\cap\mcal{E}_t^\delta(\theta_t^p)$ is needed for the algorithm. To this end, we explicitly introduce $\tilde\theta_t$ in the program via a slack variable. Formally, we study:
\begin{align}
	\begin{pmatrix} \tilde{\theta}_t\\\theta_t^p  \end{pmatrix} \in\argmin_{\theta' \in 
	\mbb{R}^d, \theta\in\mbb{R}^d}\left\{ \left\lVert g_t(\theta)-g_t(\hat\theta_t) \right\rVert_{\mathbf{V^{-2}_t}} \text{ s.t } \theta'\in\mcal{E}_t^\delta(\theta)\cap\Theta\right\}\tag{\textbf{P1'}} \;.
	\label{opt:program_1}
\end{align}
We also introduce the following program:
\begin{align}
		\begin{pmatrix} \tilde{\theta}_t\\ \eta  \end{pmatrix} \in\argmin_{\theta' \in 
		\mbb{R}^d,\eta \in\mbb{R}^d}\left\{ \left\lVert g_t(\theta')+\beta_t(\delta)	\mathbf{\widetilde{V}^{1/2}_t}\eta-g_t(\hat\theta_t) \right\rVert_{\mathbf{V^{-2}_t}} \text{ s.t } \ltwo{\theta'}\leq S, \ltwo{\eta}\leq 1\right\}\tag{\textbf{P2}} \;.
		\label{opt:program_2}
\end{align}
We claim and prove the following result, which is an equivalent reformulation of Proposition~\ref{prop:program1}.\begin{restatable}{prop}{propoptimequivalent}
	The programs \eqref{opt:program_1} and  \eqref{opt:program_2} are equivalent.
\end{restatable}
\begin{proof}
	The proof consists in building a bijection between the solutions of \eqref{opt:program_1} and \eqref{opt:program_2}. Let us introduce the mapping:
	\begin{align*}
		f\, : \, \Theta \times \mbb{R}^d &\to  \Theta \times \mbb{R}^d\\
				\begin{pmatrix}x\\ y\end{pmatrix}&\to\begin{pmatrix}f_1(x)\\ f_2(x,y)\end{pmatrix} = \begin{pmatrix}x \\\beta_t^{-1}(\delta)\mathbf{\widetilde{V}^{-1/2}_t}(g_t(y)-g_t(x))\end{pmatrix} 
	\end{align*}
We now claim the following Lemma, which proof is deferred to Section~\ref{sec:gt_bijective}.
	\begin{restatable}{lemma}{gtbijective}
	The function:
	\begin{align*}
		g_t\, : \, &\mbb{R}^d \to \mbb{R}^d\\
		& \theta \to \sum_{s=1}^{t-1}\gamma^{t-1-s}\mu(\langle \theta, x_s\rangle)x_s + \lambda\cm\theta
	\end{align*}
	is a bijection.
\end{restatable}%
A straight-forward implication of this Lemma is the \underline{bijectivity} of $f$. Let $(\tilde\theta^1,\theta^p)$ be a solution of \eqref{opt:program_1} and let:
\begin{align*}
	\begin{pmatrix} \tilde\theta^2\\ \eta^p\end{pmatrix} = f\begin{pmatrix} \tilde\theta^1\\\theta^p \end{pmatrix} \;.
\end{align*}
We are going to show that $(\tilde\theta^2,\eta^p)$ is a solution of \eqref{opt:program_2}. Because $(\tilde\theta^1,\theta^p)$ is optimal for \eqref{opt:program_1}, we have that:
\begin{align*}
	\lVert g_t(\theta^p)-g_t(\hat\theta_t)\rVert_{\mathbf{V^{-2}_t}}  &\leq \lVert g_t(\theta)-g_t(\hat\theta_t)\rVert_{\mathbf{V^{-2}_t}}&\\
	 &\qquad \forall(\theta',\theta)\in\Theta\times\mbb{R}^d \text{ s.t } \theta'\in\mcal{E}_t^\delta(\theta)\\
	\pmb{\Leftrightarrow} \lVert g_t(\theta^p)-g_t(\hat\theta_t)\rVert_{\mathbf{V^{-2}_t}} & \leq \lVert g_t(\theta)-g_t(\hat\theta_t)\rVert_{\mathbf{V^{-2}_t}}&(\text{definition of }\mcal{E}_t^\delta(\theta))\\
	& \qquad\forall(\theta',\theta)\in\Theta\times\mbb{R}^d \text{ s.t } {\color{red}\lVert g_t(\theta')-g_t(\theta)\rVert_{\mathbf{\widetilde{V}^{-1}_t}}\leq \beta_t(\delta)}&\\
	\pmb{\Leftrightarrow} \lVert g_t(\theta^p)-g_t(\hat\theta_t)\rVert_{\mathbf{V^{-2}_t}} & \leq \lVert g_t(\theta)-g_t(\hat\theta_t)\rVert_{\mathbf{V^{-2}_t}}&\\
	& \qquad \forall(\theta',\theta)\in\Theta\times\mbb{R}^d \text{ s.t } {\color{red}\ltwo{f_2(\theta',\theta)}\leq 1} &
\end{align*} 

Noticing that for all $(x,y)\in\Theta\times\mbb{R}^d$ we have $g_t(y) = g_t(x) + \beta_t(\delta)\mathbf{V^{1/2}_t}f_2(x,y)$ we therefore obtain:
\begin{align*}
	 \lVert g_t(\tilde\theta^1) + {\color{red}\beta_t(\delta)\mathbf{\widetilde{V}^{1/2}_t}f_2(\tilde\theta^1,\theta^p)} -g_t(\hat\theta_t)\rVert_{\mathbf{V^{-2}_t}} & \leq \lVert g_t(\theta') + {\color{red}\beta_t(\delta)\mathbf{\widetilde{V}^{1/2}_t}f_2(\theta',\theta)}-g_t(\hat\theta_t)\rVert_{\mathbf{V^{-2}_t}}&\\
	& \qquad \forall(\theta',\theta)\in\Theta\times\mbb{R}^d \text{ s.t } \ltwo{f_2(\theta',\theta)}\leq 1 &\\
	\pmb{\Leftrightarrow}  \lVert g_t(\tilde\theta^1) + \beta_t(\delta)\mathbf{\widetilde{V}^{1/2}_t}{\color{red}\eta^p} -g_t(\hat\theta_t)\rVert_{\mathbf{V^{-2}_t}} & \leq \lVert g_t(\theta') + \beta_t(\delta)\mathbf{\widetilde{V}^{1/2}_t}f_2(\theta',\theta)-g_t(\hat\theta_t)\rVert_{\mathbf{V^{-2}_t}}&\\
	& \qquad \forall(\theta',\theta)\in\Theta\times\mbb{R}^d \text{ s.t } \ltwo{f_2(\theta',\theta)}\leq 1 &\\
	\pmb{\Leftrightarrow}  \lVert g_t({\color{red}\tilde\theta^2}) + \beta_t(\delta)\mathbf{\widetilde{V}^{1/2}_t}\eta^p -g_t(\hat\theta_t)\rVert_{\mathbf{V^{-2}_t}} & \leq \lVert g_t(\theta') + \beta_t(\delta)\mathbf{\widetilde{V}^{1/2}_t}f_2(\theta',\theta)-g_t(\hat\theta_t)\rVert_{\mathbf{V^{-2}_t}}&\\
	& \qquad \forall(\theta',\theta)\in\Theta\times\mbb{R}^d \text{ s.t } \ltwo{f_2(\theta',\theta)}\leq 1 &(\tilde\theta^1=\tilde\theta^2)\\
	\pmb{\Leftrightarrow}  \lVert g_t(\tilde\theta^2) + \beta_t(\delta)\mathbf{\widetilde{V}^{1/2}_t}\eta^p -g_t(\hat\theta_t)\rVert_{\mathbf{V^{-2}_t}} & \leq \lVert g_t(\theta') + \beta_t(\delta)\mathbf{\widetilde{V}^{1/2}_t}f_2(\theta',\theta)-g_t(\hat\theta_t)\rVert_{\mathbf{V^{-2}_t}}&\\
	& \qquad \forall(\theta',\theta)\text{ s.t } \ltwo{f_2(\theta',\theta)}\leq 1, {\color{red}\ltwo{\theta'}\leq S}\\
	\pmb{\Leftrightarrow}  \lVert g_t(\tilde\theta^2) + \beta_t(\delta)\mathbf{\widetilde{V}^{1/2}_t}\eta^p -g_t(\hat\theta_t)\rVert_{\mathbf{V^{-2}_t}} & \leq \lVert g_t(\theta') + \beta_t(\delta)\mathbf{\widetilde{V}^{1/2}_t}{\color{red}{\eta}}-g_t(\hat\theta_t)\rVert_{\mathbf{V^{-2}_t}}&\\
	& \qquad \forall(\theta',{\color{red}{\eta}})\text{ s.t } \ltwo{{\color{red}{\eta}}}\leq 1,\ltwo{\theta'}\leq S
\end{align*}
where we last used the fact that $f_2$ spans $\mbb{R}^d$ (surjectivity). Finally, we have that:
\begin{align*}
	\ltwo{\tilde\theta^2}&\leq S &(\tilde\theta^2=\tilde\theta^1\in\Theta)\\
	\ltwo{\eta^p} &= \beta_t^{-1}(\delta)\left\lVert  g_t(\theta^p)-g_t(\tilde\theta^1)\right\rVert_{\mathbf{V^{-1}_t}} \leq 1  &(\tilde\theta^1\in\mcal{E}_t^\delta(\theta^p))
\end{align*}
Combining the last two results proves that $(\tilde\theta^2,\eta^p)$ is feasible for \eqref{opt:program_2}, and optimal within the feasible set. As a consequence, $(\tilde\theta^2,\eta^p)$ \underline{is a solution of} \eqref{opt:program_2}. Therefore, $f$ is a bijection between the minimizers of  \eqref{opt:program_1} and \eqref{opt:program_2}, which concludes the proof.
\end{proof}

\subsection{Bijectivity of $g_t$}
\label{sec:gt_bijective}
\gtbijective*

\begin{proof}
	\underline{Injectivity.} Notice that $\forall \theta\in\mbb{R}^d$:
	\begin{align*}
		\nabla_\theta g(\theta) = \sum_{s=1}^{t-1} \gamma^{t-1-s} \dot{\mu}(\langle \theta, x_s\rangle)x_sx_s^\transp + \lambda\cm\mathbf{I_d} \succ 0 \;.
	\end{align*}
	Hence $\nabla_\theta g$ is P.S.D, and a simple integral Taylor expansion is enough to prove injectivity. 

	\noindent\underline{Surjectivity} Let $z\in\mathbb{R}^d$. Let $A=\text{Span}(x_1,..,x_{t-1})$ be the vectorial space spanned by $\{x_s\}_{s=1}^{t-1}$. Let $z_\perp$ be the orthogonal projection of $z$ on $A$ and $z_\parallel = z-z_\perp$.  Since $z_\perp\in A$, there exists $\{\alpha_{s}\}_{s=1}^{t-1}\in\mathbb{R}^{t-1}$ such that:
	\begin{align*}
		z_\perp = \sum_{s=1}^{t-1} \alpha_sx_s \;.
	\end{align*}
	 Recall that $b(\cdot)$ is a primitive of $\mu$, which is convex since $\mu$ is strictly increasing. Define:
	 \begin{align*}
	 	L(\theta) = \sum_{s=1}^{t-1} \gamma^{t-1-s}\left[b(\langle \theta, x_s \rangle) - \frac{\alpha_s}{\gamma^{t-1-s}} \langle \theta, x_s \rangle\right] + \frac{\lambda\cm}{2}\left\lVert \theta - \frac{z_\parallel}{\lambda\cm}\right\rVert^2\;.
	 \end{align*}
	 which is a strictly convex, coercive function. Its minimum $\theta_z$ (which therefore exists and is uniquely defined) checks:
	 \begin{align*}
	 	&\nabla_\theta L(\theta_z) = 0 \\
		&\pmb{\Leftrightarrow} \sum_{s=1}^{t-1} \gamma^{t-1-s} \left[\mu(\langle \theta_z, x_s\rangle)-\frac{\alpha_s}{\gamma^{t-1-s}}\right]x_s + \lambda\cm
		\left(\theta_z-\frac{z_\parallel}{\lambda\cm}\right) = 0\\
		& \pmb{\Leftrightarrow} g(\theta_z)  = \sum_{s=1}^{t-1}\alpha_s x_s + z_\parallel \\
		& \pmb{\Leftrightarrow} g(\theta_z)  = z_\perp + z_\parallel = z\;.
	 \end{align*}
	 which proves surjectivity.
\end{proof}

\section{Useful lemmas}
\label{app:useful}
The following Lemma is a version of the Elliptical Potential Lemma for weighted sums, similar to Proposition 4 of \cite{russac2019weighted}. 
\begin{lemm}
    Let $\{x_s\}_{s=1}^\infty$ a sequence in $\mbb{R}^d$ such that $\ltwo{x_s}\leq L$ for all $s\in\mbb{N}^*$, and  let $\lambda$ be a non-negative scalar. For $t\geq 1$ define $\mbold{V_t} \defeq \sum_{s=1}^{t-1}\gamma^{t-1-s} x_sx_s^T+\lambda\mbold{I_d}$. The following inequality holds:
    $$
        \sum_{t=1}^{T}\lVert x_t\rVert^2_{\mathbf{V_t^{-1}}} \leq 2\max(1,L^2/\lambda )\left(dT\log(1/\gamma) +\log\left(\frac{\det\mathbf{V_{T+1}}}{\lambda^d}\right)\right)\;.
   $$
\label{lemma:ellipticalpotential}
\end{lemm}
\begin{proof}
	For all $t\geq 1$, by definition:
	\begin{align*}
		\mathbf{V_{t+1}} &=  \sum_{s=1}^{t}\gamma^{t-s} x_sx_s^\transp+\lambda\mbold{I_d} & \\
		&= \gamma  \sum_{s=1}^{t-1}\gamma^{t-1-s} x_sx_s^T+ x_tx_t^\transp+\lambda\mbold{I_d}\\
		&\succeq \gamma\left(\sum_{s=1}^{t-1}\gamma^{t-1-s} x_sx_s^T+ x_tx_t^\transp+\lambda\mbold{I_d}\right) &(\gamma\leq 1)\\
		&\succeq  \gamma\left(\mathbf{V_t}+ x_tx_t^\transp\right) &\\
		&\succeq  \gamma\mathbf{V_t}\left(\mathbf{I_d}+ \mathbf{V_t^{-1/2}}x_tx_t^\transp\mathbf{V_t^{-1/2}}\right) \;,
	\end{align*}
which after some easy manipulations yields:
\begin{align*}
	d\log(1/\gamma) + \log \det\mathbf{V_{t+1}}-\log \det\mathbf{V_{t}} \geq \log\left(1+\lVert x_t\rVert^2_{\mathbf{V_t^{-1}}}\right)\;.
\end{align*}
After summing from $t=1$ to $t=T$ and telescoping we obtain:
\begin{align*}
	dT\log(1/\gamma) +\log\left(\frac{\det\mathbf{V_{T+1}}}{\lambda^d}\right) &\geq  \sum_{t=1}^{T}\log\left(1+\lVert x_t\rVert^2_{\mathbf{V_t^{-1}}}\right)&\\
	&\geq  \sum_{t=1}^{T}\log\left(1+ \frac{1}{\max(1,L^2/\lambda )}\lVert x_t\rVert^2_{\mathbf{V_t^{-1}}}\right) \;.
\end{align*}
Finally, noticing that $\frac{1}{\max(1,L^2/\lambda )}\lVert x_t\rVert^2_{\mathbf{V_t^{-1}}}\leq 1$ and using the fact that for all $x\in(0,1]$ we have $\log(1+x)\geq x/2$ we obtain: 
\begin{align*}
	dT\log(1/\gamma) +\log\left(\frac{\det\mathbf{V_{T+1}}}{\lambda^d}\right) &\geq \frac{1}{2\max(1,L^2/\lambda )} \sum_{t=1}^{T}\lVert x_t\rVert^2_{\mathbf{V_t^{-1}}} \;,
\end{align*}
which in turn yields:
\begin{align*}
	\sum_{t=1}^{T}\lVert x_t\rVert^2_{\mathbf{V_t^{-1}}} \leq 2\max(1,L^2/\lambda )\left(dT\log(1/\gamma) +\log\left(\frac{\det\mathbf{V_{T+1}}}{\lambda^d}\right)\right)
	\;,
\end{align*}
which is the announced result.
\end{proof}

We also remind here the determinant-trace inequality for the weighted design matrix which can be extracted from Proposition~2 of \cite{russac2019weighted}.
\begin{lemm}
     Let $\{x_s\}_{s=1}^\infty$ a sequence in $\mbb{R}^d$ such that $\ltwo{x_s}\leq L$ for all $s\in\mbb{N}^*$, and  let $\lambda$ be a non-negative scalar. For $t\geq 1$ define $\mbold{V_t} \defeq \sum_{s=1}^{t-1}\gamma^{t-1-s} x_sx_s^T+\lambda\mbold{I_d}$. The following inequality holds:
     \begin{align*}
         \det(\mbold{V_{t+1}}) \leq \left(\lambda+\frac{L^2(1-\gamma^{t})}{d(1-\gamma)}\right)^d \;.
     \end{align*}
\label{lemma:determinant_trace_inequality}
\end{lemm}

\section{$\ouralgo$ algorithm}
\label{app:bob}

\subsection{High-level ideas}
In this part of the appendix, we denote $\gamma^\star$ as follows:
\begin{equation}
    \label{eq:gamma_star_app}
    \gamma^\star = 1 - \frac{1}{2} \left(\frac{B_{T,\star}}{d T 2 S} \right)^{2/3} \;.
\end{equation}

\begin{remark}
$\gamma^\star$ as defined in Equation \eqref{eq:gamma_star_app} has a different expression than the 
discount factor
proposed in Theorem \ref{thm:regret_bound}. This slight modification is to ensure that $\gamma^\star$ is larger than 1/2 
and simplifies the finite time analysis of the regret.
Yet, it has no consequence on the asymptotic bound. 
\end{remark}


$B_{T,\star}$ being unknown, we cannot compute the optimal discount factor 
that depends on the parameter drift.
The general idea is to use a set of different values for the discount factor 
(respectively the $B_{T,\star}$ values)
called 
$\mathcal{H}$, covering the $[1/2, 1)$ space (respectively the $[0,2ST)$ space). 
Then, we divide the time horizon $T$ into 
different blocks of length $H$. Every $H$ steps, we create \textbf{a new instance} 
of $\ouralgo$ with 
a $\gamma$ that is chosen by a \textit{master} algorithm: the $\EXP$ algorithm from 
\cite{auer2002nonstochastic}. 
At the end of each block, this \textit{master} algorithm receives the cumulative rewards
from 
the instantiated \textit{worker} and updates its probability 
distribution over the set $\mathcal{H}$. The objective of the master algorithm is to learn 
the most suitable value of $\gamma$ so as to maximise the cumulative rewards in
accordance with the dynamics of the environment. On the other side, the different 
\textit{workers} algorithms act exactly as if the $\ouralgo$ algorithm
was launched on a $H$-steps experiment. This setting is similar
to the one presented in \cite{cheung2019hedging} (respectively \cite{zhao2020simple}) with
discount factors instead of sliding windows (respectively restart parameters). 
This framework is called Bandit-over-Bandit (BOB) precisely because of this two-stage
structure between the \textit{master} and the \textit{workers} algorithms.

\subsection{Algorithm}
\label{app:bob_algo}
The coverage $\cH$ with the different discount factors is defined in the following way:
\begin{align}
\label{eq:setH}
\cH &= \{\gamma_i = 1-\mu_i |i=1,\dots,N\}\;  \\
\text{with} \; N =& \left\lceil
\frac{2}{3}\log_2\left(2ST^{3/2} \right) \right\rceil + 1 \;
\text{and} 
\; \mu_i = \frac{1}{2}\frac{2^{i-1}}{ d^{2/3} T(2S)^{2/3}} \;.
\label{eq:N}
\end{align}
The \textit{main} algorithm is an instance of the $\EXP$ algorithm from
\cite{auer2002nonstochastic} where the different arms correspond to the different
discount factors. Following $\EXP$ 
analysis \citep{auer2002nonstochastic}, the probability of drawing $\gamma_j$ for the block $i$ is 
\begin{equation}
\label{eq:exp3_prob}
p_i^{\gamma_j} = (1-\alpha)\frac{s_i^{\gamma_j}}{\sum_j s_i^{\gamma_j}}+\frac{\alpha}{N} ,
\;
\forall j = 1,2,\dots, N \;,
\end{equation}
where $\alpha$ is defined as
\begin{equation}
\label{eq:exp3_alpha}
\alpha = \min\Bigg\{1, \sqrt{\frac{N\log(N)}{(e-1)\lceil T/H\rceil}}\Bigg\} \;
\end{equation}
and $s_i^{\gamma_j}$ is initialised at $1$ and is updated at the end of each block \textbf{when selected} with
\begin{equation}
\label{eq:exp3_s}
s_{i+1}^{\gamma_j} = s_i^{\gamma_j} 
\exp\Bigg(\frac{\alpha}{Np_i^{\gamma_j}}  \frac{\sum_{t=(i-1)H+1}^{\min\{iH,T
\}}r_{t+1}}{2 \sigma H} \Bigg) \;.
\end{equation}
Note that in Equation \eqref{eq:exp3_s}, $r_{t+1}$ is the noisy reward obtained when the
action $x_t$ is selected with the $\ouralgo$ algorithm with parameter $\gamma_j$.
Equation
\eqref{eq:exp3_prob}, \eqref{eq:exp3_alpha} and \eqref{eq:exp3_s} 
are the same as in 
\cite{auer2002nonstochastic} except for the rescaling 
of the cumulative rewards on a 
block that is required to ensure that they lie in $[0,1]$. 
Details on this rescaling 
part can be found in Proposition \ref{prop:esecond}.

\begin{algorithm}
\caption{\ouralgoBOB (detailed)}
\label{alg:meta}
  \begin{algorithmic}
    \STATE{ \bfseries Input.} Length $H$, time
    horizon $T$, regularization $\lambda$, confidence $\delta$, inverse link function $\mu$, 
    constants $S,L$ and $\sigma$.
    \STATE {\bfseries Initialization.}
    Create the covering space $\mathcal{H}$ as defined in Eq. \eqref{eq:setH},
    set $s_1^{\gamma_i} = 1$, $\forall \gamma_i\in\cH$.
    \FOR{$i= 1,\ldots, \lceil T/H\rceil$}
    \STATE{$\gamma_j \sim p_i^\gamma$, the probability vector defined in 
    Eq. \eqref{eq:exp3_prob}.} \\
    \STATE{Start a $\ouralgo$ subroutine with parameter $\gamma_j$} \\
    \FOR{$t = (i-1)H+1, \dots, \min\{i H,T\}$}
    \STATE{Receive the action set $\mathcal{X}_t$.}
    \STATE{Select $x_t(\gamma_j) \in \mathcal{X}_t$ with $\ouralgo$.}
    \STATE{Observe reward $r_{t+1}$.}
    \ENDFOR
    \STATE{Update $s_{i+1}^{\gamma_j}$ according to
    Equation~(\ref{eq:exp3_s}).}
    \STATE{Update $s_{i+1}^{\gamma} = s_i^{\gamma}$, $\forall \gamma 
    \neq \gamma_j$.}
    \ENDFOR
  \end{algorithmic}
\end{algorithm}

\begin{remark}
We denote $x_t(\gamma)$ the action chosen with the $\ouralgo$ algorithm with a discount factor $\gamma$.
\end{remark}

\subsection{Regret guarantees}

In this section, we give an upper-bound for the expected dynamic regret
of $\ouralgoBOB$. By construction, it is natural to decompose the regret 
into two sources of errors. First the \textit{master} error 
committed by the $\EXP$ algorithm by not choosing the best possible discount factor.
Second the \textit{worker} error inherent to the $\ouralgo$ algorithm.
Note that there are two independent sources of randomness: the
stochasticity of the rewards (whose expectation is denoted $\EN$) and the
randomness
of the $\EXP$ algorithm (denoted $\EEXP$). Bringing things together, 

\begin{equation}
\label{eq:regret_decom}
\begin{array}{ll}
\mathbb{E}[R_T] &= \EN \left[\sum\limits_{t=1}^T  \mu( \langle x_\star^t,\theta_\star^t \rangle) - 
\EEXP[r_{t+1}]
\right] \\
& = \underbrace{\EN \left[\sum\limits_{t=1}^T \mu( \langle x_\star^t,\theta_\star^t \rangle)
     - \sum\limits_{i=1}^{\lceil T/H\rceil}
     \sum\limits_{t=(i-1)H+1}^{\min\{iH,T\}}
     \mu(\langle x_t(\widehat{\gamma}),\theta^t_\star
     \rangle) \right]}_{\efirst} \\
  &\quad   
     + \quad  \underbrace{\EN \left[\sum\limits_{i=1}^{\lceil T/H\rceil} \sum\limits_{t=(i-1)H+1}^{\min\{iH,T\}}
     \mu(\langle x_t(\widehat{\gamma}),\theta_\star^t) \rangle 
     - \EEXP \left[ r_{t+1} \right] \right]}_{\esecond} \;.
\end{array}
\end{equation}


The next step consists in upper-bounding the $\efirst$ error and the $\esecond$ error from
Eq.~\eqref{eq:regret_decom} respectively.

\begin{lemm}
\label{lemma:gamma_k}
With pavement $\mathcal{H}$ defined in Equation \eqref{eq:setH} for any 
unknown $B_{T,\star} > 0$, setting 
$k=\lfloor\frac{2}{3}\log_2(B_{T, \star} T^{1/2})\rfloor+1$ yields
$$
\gamma_{k+1} \leq \gamma^\star \leq \gamma_k \;.
$$
\end{lemm}

\begin{proof}
With assumption \ref{ass:bounded_decision_set}, we have $B_{T, \star} \leq 2S T$. 
Using this, $k$ (as defined
in the statement of the lemma) is smaller than $N$.
We have,
\begin{align*}
    & k-1 \leq \frac{2}{3} \log_2(B_{T,\star } T^{1/2}) \leq k \\
    \pmb{\Leftrightarrow} &  -\frac{1}{2} \frac{2^{k-1}}{d^{2/3}T (2S)^{2/3}}
    \geq -\frac{1}{2} \left(\frac{B_{T, \star} }{d T 2S}\right)^{2/3}
    \geq -\frac{1}{2} \frac{2^{k}}{d^{2/3}T (2S)^{2/3}} \;.
\end{align*}
Adding one for the different terms gives the result.
\end{proof}
For the rest of the section, we set $\gammae = \gamma_k$ with $k$ defined in Lemma \ref{lemma:gamma_k}.
We denote $B_{i,\star} = \sum_{t= (i-1)H +1}^{i H-1} \lVert \theta_\star^{t+1}- \theta_\star^t
\rVert_2$
and
\begin{equation}
    \label{eq:beta_H_star}
\beta_H^\star = \sqrt{\lambda} S + \sigma \sqrt{2\log(T) + d \log \left(1 + \frac{2L^2}{\lambda d (1- \gamma^{\star 2} )}
\right)
} \; .
\end{equation}

\begin{prop}
\label{proposition:efirst}
The $\efirst$ error can be upper-bounded in the following way:
\begin{equation*}
\begin{split}
\efirst &\leq 2 \sigma \frac{T}{H} + C_1 R_\mu \beta_H^\star  \sqrt{d T} 
\sqrt{
2 T (1- \gamma^\star) + \frac{T}{H} \log\left(1 + \frac{2 L^2}{d \lambda (1- \gamma^\star)} \right) 
     } \\
& \quad     
+ \;
     2 C_2 R_\mu \frac{1}{\sqrt{T}} \frac{1}{1- \gamma^\star}
+ \frac{3 C_3 R_\mu }{\log(2)} \frac{B_{T, \star} \log(T)}{1- \gamma^\star} \;,
\end{split}
\end{equation*}
with $C_1$, $C_2$, $C_3$ constant terms from Theorem \ref{thm:regret_bound} and $\beta_H^\star$ defined in Equation 
\eqref{eq:beta_H_star}.
\end{prop}
\begin{proof}
First, note that our objective here is to bound the expected regret whereas
Theorem \ref{thm:regret_bound}
bounds the pseudo-regret and gives a high probability upper-bound.
We denote $E^i_\delta = \{ \bar{\theta}_t \in \mathcal{E}_t^\delta(\hat{\theta_t})
\; \textnormal{for} \; t\; \textnormal{s.t} \; (i-1)H + 1 \leq t \leq \min \{iH,T\} 
\}$. This event holds with probability higher than $1- \delta$.
When $E^i_\delta$ does not hold, 
the maximum regret could theoretically be suffered for all time
instants. 

As explained in the algorithm mechanism, a new instance of $\ouralgo$ will be
launched every $H$ steps
with a discount factor selected by the $\EXP$ algorithm. Restarting a new
algorithm and 
forgetting previous information comes at a cost in terms of regret. This is
made explicit
in the following decomposition of $\efirst$.
\begin{align*}
\efirst &= \EN \left[ \sum\limits_{i=1}^{\lceil
T/H\rceil}\sum\limits_{t=(i-1)H+1}^{\min\{iH,T\}}
\mu ( \langle x_\star^t, \theta_\star^t \rangle ) - 
\mu ( \langle x_t(\widehat{\gamma}), \theta_\star^t \rangle ) \right] \\ 
&= \underbrace{\EN \left[ \sum\limits_{i=1}^{\lceil
T/H\rceil}\sum\limits_{t=(i-1)H+1}^{\min\{iH,T\}}\langle
\mu ( \langle x_\star^t, \theta_\star^t \rangle ) - 
\mu ( \langle x_t(\widehat{\gamma}), \theta_\star^t \rangle )
\Big| \{ \cap_{i=1}^{\lceil
T/H\rceil} E^i_\delta
\} \right] \mathbb{P}\left( \cap_{i=1}^{\lceil
T/H\rceil} E^i_\delta \right)}_{worker_1} \\
& + \underbrace{\EN \left[ \sum\limits_{i=1}^{\lceil
T/H\rceil}\sum\limits_{t=(i-1)H+1}^{\min\{iH,T\}}
\mu ( \langle x_\star^t, \theta_\star^t \rangle ) - 
\mu ( \langle x_t(\widehat{\gamma}), \theta_\star^t \rangle ) \Big|
\{ \cap_{i=1}^{\lceil
T/H\rceil} E^i_\delta  \}^c \right] \mathbb{P}\left( \{ \cap_{i=1}^{\lceil
T/H\rceil} E^i_\delta  \}^c \right)}_{worker_2}
\end{align*}
Thanks to Lemma \ref{lemma:confidence_set}, 
$E^i_\delta$ holds with probability higher than 1-$\delta$. By setting 
$\delta = 1/T$, we have

\begin{align}
\label{eq:proba_e_i_c}
     \mathbb{P}\left( \cup_{i=1}^{\lceil
T/H\rceil} (E^i_\delta )^c \right) \leq  \lceil T/H\rceil 1/T \;.
\end{align}
Under the event $\{ \cup_{i=1}^{\lceil
T/H\rceil} (E^i_\delta)^c \}$ not much can be said.
The maximum regret $r_{\textnormal{max}} = 2\sigma$ can be suffered at every time
step.
Therefore, using the upper-bound from Eq. \eqref{eq:proba_e_i_c}, we obtain
\begin{align*}
  worker_2 &=  \EN \left[ \sum\limits_{i=1}^{\lceil
T/H\rceil}\sum\limits_{t=(i-1)H+1}^{\min\{iH,T\}}
\mu ( \langle x_\star^t, \theta_\star^t \rangle ) - 
\mu ( \langle x_t(\widehat{\gamma}), \theta_\star^t \rangle )
 \Big|
\{  \cup_{i=1}^{\lceil
T/H\rceil} (E^i_\delta )^c
\} \right] \mathbb{P}\left(  \cup_{i=1}^{\lceil
T/H\rceil} (E^i_\delta )^c \right)  \\
&\leq  r_{\textnormal{max}} \lceil T/H\rceil \;.
\end{align*}
This term is related to the number of restarts of the algorithm.
In the BOB framework, whatever the worker algorithm (sliding window, 
restart factor) a cost of order $T/H$ will be paid due to 
the restarting of the \textit{worker} at the beginning of each block.

On the contrary, under the event 
$\{ \cap_{i=1}^{\lceil
T/H\rceil}E^i_\delta 
\}$,
using the assumption that the blocks are independent, we can follow the 
line of proof from Lemma \ref{lemma:regret_decomposition} and Theorem
\ref{thm:regret_bound} for every block. We introduce,
\begin{equation}
    \label{eq:beta_H}
\beta_H = \sqrt{\lambda} S + \sigma \sqrt{2\log(T) + d \log \left(1 + \frac{L^2 (1-\gamma_k^{2H})}{\lambda d (1- \gamma_k^2)}
\right)
} \; .
\end{equation}

\begin{align*}
    \efirst_1 &= \EN \left[ \sum\limits_{i=1}^{\lceil
T/H\rceil}\sum\limits_{t=(i-1)H+1}^{\min\{iH,T\}}
\mu ( \langle x_\star^t, \theta_\star^t \rangle ) - 
\mu ( \langle x_t(\widehat{\gamma}), \theta_\star^t \rangle )
\Big|
\{ \cap_{i=1}^{\lceil
T/H\rceil}E^i_\delta 
\} \right] \mathbb{P}\left( \cap_{i=1}^{\lceil
T/H\rceil} E^i_\delta \right) 
\\
& \leq \EN \left[ \sum\limits_{i=1}^{\lceil
T/H\rceil}\sum\limits_{t=(i-1)H+1}^{\min\{iH,T\}}
\mu ( \langle x_\star^t, \theta_\star^t \rangle ) - 
\mu ( \langle x_t(\widehat{\gamma}), \theta_\star^t \rangle )
\Big|
\{ \cap_{i=1}^{\lceil
T/H\rceil} E^i_\delta 
\} \right] 
\\
&\leq  \sum_{i=1}^{\lceil
T/H\rceil} 
\left(
C_1 \beta_H \sqrt{dH}\sqrt{H\log(1/\gammae)+ 
\log\left(1 + \frac{L^2}{d \lambda 
(1- \gammae)}
\right)} +
 C_2\frac{\gammae^{D}}
{1-\gammae} H 
+ 
C_3 B_{i, \star} D
\right)
\\
& \leq 
C_1 \beta_H \sqrt{dT}\sqrt{T\log(1/\gammae)+
\frac{T}{H}\log\left(1 + \frac{L^2}{d \lambda 
(1- \gammae)}
\right)}
+ C_2\frac{\gammae^{D}}{1-\gammae}T
+ C_3 B_{T, \star} D
 \;,
\end{align*}
where the second inequality is a consequence of Theorem \ref{thm:regret_bound}.
We set, 
\begin{equation}
  \label{eq:D}
   D = \frac{3/2
\log(T)}{\log(1/\gammae)} \;.
\end{equation}

Hence, 
\begin{align*}
    C_3 B_{T, \star} D &\leq \frac{3}{2}\frac{ C_3 B_{T, \star} \log(T)}{\log(1/\gammae)} 
    \\
    &\leq \frac{3 C_3}{2 \log(2)} B_{T, \star} \log(T)
    \frac{\gammae}{1- \gammae} \quad (\textnormal{Using } \log(x) \geq
    \log(2)(x-1) \textnormal{ for } x \in [1,2])
    \\
    &\leq \frac{3 C_3}{ 2 \log(2)} \frac{ B_{T, \star} \log(T)}{1-\gamma_k}  \quad 
    (\gammae \leq 1)
    \\
    & \leq \frac{3 C_3}{\log(2)} \frac{ B_{T, \star}  \log(T)}{1-\gamma_{k+1}} \quad 
    (\textnormal{Definition of } \mathcal{H})
    \\
    & \leq \frac{3 C_3}{\log(2)} 
    \frac{ B_{T, \star} \log(T)}{1-\gamma^\star} \quad (\textnormal{Lemma } \ref{lemma:gamma_k})\;.
\end{align*}

We also have, 
\begin{align*}
    C_2 \frac{\gammae^{D}}{1-\gammae}T & \leq
    C_2 \frac{1}{\sqrt{T}} \frac{1}{1- \gammae} \quad (\textnormal{Equation } 
    \eqref{eq:D}) \\
    & \leq 2 C_2  \frac{1}{\sqrt{T}} \frac{2}{1- \gamma_{k+1}} 
    \quad (\textnormal{Definition of } \mathcal{H})
    \\
    &\leq 2 C_2 \frac{1}{\sqrt{T}} \frac{1}{1- \gamma^\star} 
    \quad (\textnormal{Lemma } \ref{lemma:gamma_k}) \;.
\end{align*}

Finally, using $ x \mapsto \log(x) \leq x-1$ for $x>1$ and Lemma \ref{lemma:gamma_k},
one has:
\begin{align*}
    T \log(1/\gammae) + \frac{T}{H} \log\left(1 + \frac{L^2}{d \lambda 
(1- \gammae)}
\right) & \leq T \frac{1- \gammae}{\gammae} + \frac{T}{H} \log\left(1 + \frac{2 L^2}{d \lambda 
(1- \gamma^\star)} \right) 
\\
& \leq 2 T (1-\gamma^\star) + \frac{T}{H} \log\left(1 + \frac{2 L^2}{d \lambda 
(1- \gamma^\star)} \right) \;. 
\end{align*}

Following similar steps, we can upper-bound $\beta_H$ from Equation \eqref{eq:beta_H} by
$$
\beta_H \leq \beta_H^\star \;.
$$
Bringing things together, we have shown that under the event 
$\{ \cap_{i=1}^{\lceil
T/H\rceil}\mathcal{E}_i 
\}$
all the terms depending 
on $\gammae$ can be replaced by terms depending only on $\gamma^\star$ at
the cost of multiplicative constant independent of $T$. 
Finally, one has
\begin{equation*}
\begin{split}
\efirst &\leq 2 \sigma \frac{T}{H} + C_1 R_\mu \beta_H^\star  \sqrt{d T} 
\sqrt{
2 T (1- \gamma^\star) + \frac{T}{H} \log\left(1 + \frac{2 L^2}{d \lambda (1- \gamma^\star)} \right) 
     } \\
& \quad     
+ \;
     2 C_2 R_\mu \frac{1}{\sqrt{T}} \frac{1}{1- \gamma^\star}
+ \frac{3 C_3 R_\mu }{\log(2)} \frac{ B_{T, \star} \log(T)}{1- \gamma^\star} \;.
\end{split}
\end{equation*}
\end{proof}
The above proposition bounds the regret incurred if the same discount factor
$\gammae$ is used for each block.
To successfully upper bound $\ouralgo$'s regret,
we need to upper bound the second part $\esecond$ which is the error due to the use of the 
$\EXP$ algorithm. 
This part can be controlled thanks to the analysis proposed in
\cite{auer2002nonstochastic}. Yet, two issues need
to be overcome. (1) The rewards received at the end
of a block does not lie in $[0,1]$ 
which is required to use the result from \cite{auer2002nonstochastic}.
(2) We are in a stochastic environment with noisy rewards.

In the next proposition, we upper-bound the term of interest and explain
how to deal with the two issues. The big picture is the following:
using the assumption on the bounded rewards we can obtain  
an upper-bound for the maximum reward on a single block.
\begin{prop}
\label{prop:esecond}
The regret due to the \esecond{} algorithm can be bounded in the following way,
$$
\EN \left[\sum\limits_{i=1}^{\lceil T/H\rceil}
    \sum\limits_{t=(i-1)H+1}^{\min\{iH,T\}}
    \mu( \langle x_t(\widehat{\gamma}),\theta^t_\star \rangle) 
    - \EEXP \left[ r_{t+1} \right] \right]
    \leq 4 \sigma H  \sqrt{e-1} \sqrt{\frac{T}{H} \text{card}(\mathcal{H})
    \log(\text{card}(\mathcal{H})) }
$$
\end{prop}
\begin{proof}
We denote $\gamma_i$ the discount factor chosen by the EXP3 algorithm
in the $i$-th block. The regret due to the use of 
the EXP3 \textit{main} algorithm can be written as
follows:
\begin{align*}
    \esecond 
    &= \EN \left[\sum\limits_{i=1}^{\lceil T/H\rceil}
    \sum\limits_{t=(i-1)H+1}^{\min\{iH,T\}}
     \mu( \langle x_t(\widehat{\gamma}),\theta_\star^t \rangle )
     - \EEXP \left[ \sum_{i=1}^{\lceil T/H\rceil}
     \sum\limits_{t=(i-1)H+1}^{\min\{iH,T\}}
     r_{t+1}
       \right] \right] \;.
\end{align*}

We introduce $Q_i(\gamma_j) = \sum\limits_{t=(i-1)H+1}^{\min\{iH,T\}} r_{t+1}(\gamma_j) =\sum\limits_{t=(i-1)H+1}^{\min\{iH,T\}}
\mu(\langle x_t(\gamma_j), \theta_\star^t \rangle) + \epsilon_{t+1}$, using
Equation \eqref{eq:eta_def}.
This quantity corresponds to the reward obtained on the $i$-th block 
when using $\ouralgo$ with the discount factor $\gamma_j$. We also use
$Q_i = \max_{\gamma \in \mathcal{H}} Q_i(\gamma)$.

Contrarily to existing works in the linear setting 
(e.g \cite[Lemma3]{cheung2019learning}) 
our assumption on the bounded rewards is sufficient to solve both problems.
We have, $|Q_i| \leq 2\sigma H$ almost surely using $r_t \leq 2\sigma$ for all time instants.

Let $\mathcal{U} = \{ \forall t \leq T, 0 \leq r_t \leq 2 \sigma \} $. Thanks to Assumption \ref{ass:bounded_reward}, 
we have $\mathbb{P}(\mathcal{U}) = 1$.


One has,
\begin{align*}
    \esecond &\leq \EN \left[\sum\limits_{i=1}^{\lceil T/H\rceil}
    Q_i(\gamma_k) 
     - \max_{\gamma \in \mathcal{H}} \sum\limits_{i=1}^{\lceil T/H\rceil}
     Q_i(\gamma) 
     + \max_{\gamma \in \mathcal{H}} \sum\limits_{i=1}^{\lceil T/H\rceil}
     Q_i(\gamma) 
     - \EEXP \left[ \sum_{i=1}^{\lceil T/H\rceil}
        Q_i(\gamma_i)
       \right] \right] 
       \\
    & \leq \EN \left[ \max_{\gamma \in \mathcal{H}} \sum\limits_{i=1}^{\lceil T/H\rceil}
     Q_i(\gamma) 
     - \EEXP \left[ \sum_{i=1}^{\lceil T/H\rceil}
        Q_i(\gamma_i)
       \right] \right] 
    \\
    & \leq
    \EN \left[ \max_{\gamma \in \mathcal{H}} \sum\limits_{i=1}^{\lceil T/H\rceil}
     Q_i(\gamma) 
     - \EEXP \left[ \sum_{i=1}^{\lceil T/H\rceil}
        Q_i(\gamma_i)
       \right]  \Big| \; \mathcal{U} \right] \mathbb{P}(\mathcal{U}) \; .
\end{align*}

We introduce 
$$
Y_i(\gamma_j) = \frac{Q_i(\gamma_j)}
{2 \sigma H}\;.
$$

For all $\gamma$ in $\mathcal{H}$, $Y_i(\gamma)$ lies in $[0,1]$.
Therefore,
$$
\esecond \leq 2 \sigma H \EN \left[ \max_{\gamma \in \mathcal{H}} \sum_{i =1}^
{\lceil T/H \rceil} Y_i(\gamma) - \EEXP \left[ \sum_{i =1}^
{\lceil T/H \rceil} Y_i(\gamma_i) \right] \Big| \; \mathcal{U} \right] \;.
$$

The last step consists in using \cite[Corollary 3.2]{auer2002nonstochastic}.
We have,
$$
\max_{\gamma \in \mathcal{H}} \sum\limits_{i=1}^{\lceil T/H\rceil} Y_i(\gamma) \leq \frac{T}{H} \;.
$$
All the conditions of Corollary 3.2 in \cite{auer2002nonstochastic} are met
and we obtain:
$$
\esecond \leq 4 \sigma H  \sqrt{e-1} \sqrt{\frac{T}{H} \text{card}(\mathcal{H})
    \log(\text{card}(\mathcal{H})) }  \;.
$$
\end{proof}


The two parts of regret in Equation~(\ref{eq:regret_decom}) are bounded in 
Proposition \ref{proposition:efirst} and Proposition \ref{prop:esecond} 
respectively. Combining them, we get our main result below:


\thmregretmaster*

\begin{remark}
This theorem establishes an upper-bound for the expected regret in the Generalized Linear Bandits framework when the variational budget
is unknown. When $B_{T, \star}$ is sufficiently large ($B_{T, \star} 
\geq d^{-1/2} T^{1/4}$) the obtained bound can not be improved.
Yet, there is still a gap with the lower bound when the variation budget is small. This can be explained by the 
frequent restarts in the BOB framework.
\end{remark}
\begin{proof}
Using Proposition \ref{prop:esecond} and
Proposition \ref{proposition:efirst}, we obtain: 
\begin{equation*}
\begin{split}
\mathbb{E}\left[ R_T \right] 
& \leq 2 \sigma \frac{T}{H} + C_1 R_\mu \beta_H^\star  \sqrt{d T} 
\sqrt{
2 T (1- \gamma^\star) + \frac{T}{H} \log\left(1 + \frac{2 L^2}{d \lambda (1- \gamma^\star)} \right) 
     } \\
& \quad     
+ \;
    C_2 R_\mu \frac{2}{\sqrt{T}} \frac{1}{1- \gamma^\star}
+ \frac{3 C_3 R_\mu }{\log(2)} \frac{B_{T,\star} \log(T)}{1- \gamma^\star}
+ 4 \sigma H  \sqrt{e-1} \sqrt{\frac{T}{H} \text{card}(\mathcal{H})
    \log(\text{card}(\mathcal{H})) } 
\end{split}
\end{equation*}

First note that card$(\cH) = N$ defined in Equation \eqref{eq:N}
scales as $\log(T)$ and $\beta^\star_H$ scales as $\sqrt{d \log(T)}$. 
By plugging $H = \lfloor d \sqrt{T}\rfloor$ in the upper-bound we obtain:

$$
\frac{T}{H} = \mathcal{O} (d^{-1/2} \sqrt{T}) \; . 
$$

\begin{align*}
\beta_H^\star  \sqrt{d T} 
\sqrt{
2 T (1- \gamma^\star) + \frac{T}{H} \log\left(1 + \frac{2 L^2}{d \lambda (1- \gamma^\star)} \right) 
     } 
& = \widetilde{\mathcal{O}} \left( d \sqrt{T} \sqrt{
     \max \left( \frac{T B_{T,\star}^{2/3}}{d^{2/3} T^{2/3}}, \frac{T}{d \sqrt{T}}\right)}
     \right) \\
& = d^{2/3} T^{2/3} \max(B_{T,\star}^{1/3}, d^{-1/6} T^{1/12}) \\
& = d^{2/3} T^{2/3} (\max(B_{T,\star}, d^{-1/2} T^{1/4}))^{1/3} \; .
\end{align*}

$$
\frac{1}{\sqrt{T}} \frac{1}{1- \gamma^\star} = \mathcal{O} \left(\frac{T^{1/6}}{d^{2/3} B_{T,\star}^{2/3}} \right) \;.
$$

$$
\frac{B_{T,\star}}{1- \gamma^\star} = \mathcal{O} \left( d^{2/3} B_{T,\star}^{1/3} T^{2/3} \right) \;.
$$

$$
H \sqrt{\frac{T}{H} \text{card}(\mathcal{H}) \log(\text{card}(\mathcal{H})) }  = \widetilde{\mathcal{O}} 
\left( d^{1/2} T^{3/4} \right) \;.
$$

To conclude we notice that when $B_{T,\star} \leq d^{-1/2} T^{1/4}$, 
$$
d^{1/2} T^{3/4} = d^{2/3} T^{2/3} (\max(B_{T,\star}, d^{-1/2} T^{1/4}))^{1/3} \;.
$$

On the contrary, when $B_{T,\star} \geq d^{-1/2} T^{1/4}$,
$$
d^{1/2} T^{3/4} \leq d^{2/3} T^{2/3} (\max(B_{T,\star}, d^{-1/2} T^{1/4}))^{1/3} \;.
$$
Finally, keeping the highest order term yields  the announced result.
\end{proof}

\section{Experimental set-up}
\label{app:exps}
This section is dedicated at providing useful details about the illustrative experiments presented in Section~\ref{sec:exps}. The logistic setting at hand is characterized by the constants $S=L=1$. At each round, the environment randomly draws 10 news arms, presented to the agent. All algorithms use the same $\ell_2$ regularization parameter $\lambda=1$. The sequence $\ts{t}$ evolves as follows: we let $\ts{t}= (0,1)$ for $t\in[1,T/3]$. Between $t=T/3$ and $t=2T/3$ we smoothly rotate $\ts{t}$ from $(0,1)$ to $(1,0)$. Finally  we let $\ts{t}= (0,1)$ for $t\in[2T/3,T]$. Easy computations show that the total variation budget is $$B_T=(2T/3)\sin\left(\frac{3\pi}{4T}\right)\simeq 1.5\; .$$We used the optimal value of $\gamma$ recommended by the asymptotic analysis for \texttt{D-LinUCB} and \ouralgo. We solve the projection step of \texttt{GLM-UCB} and \ouralgo{} by (constrained) gradient-based methods, thanks to the SLSQP solver of \texttt{scipy}. 

\begin{rem*}
In our experiments, we did not report performances of the algorithms from \citet{russac2020algorithms, russac2020self} (which use a similar projection step as in \citet{filippi2010parametric}). Because such algorithms are based on \emph{discrete} switches of the reward signal, their behavior in this slowly-varying environment is largely sub-optimal. Indeed, in our experiment the number of abrupt-changes is $\Gamma_T=1000$. For exponentially weighted algorithms, the recommended asymptotic value for the weights becomes $\gamma\simeq 0.70$, which in turns leads to algorithms that  over-estimate the non-stationary nature of the problem, and perform poorly in practice.
\end{rem*}

\end{document}